\documentclass{article}

\usepackage{hyperref}

\usepackage{amsmath}
\usepackage{amsthm}
\usepackage{amssymb}
\usepackage{bbm}
\usepackage{mathtools}
\usepackage{mathrsfs}
\usepackage{dsfont}
\mathtoolsset{showonlyrefs}

\newtheorem{thm}{Theorem}[]

\newtheorem{prop}{Property}[]
\newtheorem{ass}{Assumption}[]

\usepackage{courier} 

\usepackage{lipsum} 

\usepackage[space]{grffile} 

\usepackage[round]{natbib}

\usepackage{graphicx}
\usepackage{wrapfig}
\usepackage{xcolor}
\usepackage{colortbl}
\definecolor{dark-red}{rgb}{0.4,0.15,0.15}
\definecolor{dark-blue}{rgb}{0,0,0.7}
\hypersetup{
    colorlinks, linkcolor={dark-blue},
    citecolor={dark-blue}, urlcolor={dark-blue}
}

\newcommand{\topic}[1]{\textcolor{violet}{}}

\DeclareMathOperator*{\argmin}{arg\,min}

\usepackage{algpseudocode}
\usepackage[vlined,linesnumbered,ruled,algo2e]{algorithm2e}
\SetKwProg{Fn}{Function}{}{}
\let\oldnl\nl
\newcommand{\nonl}{\renewcommand{\nl}{\let\nl\oldnl}}

\usepackage{multicol}
\usepackage{enumitem}
\usepackage[utf8]{inputenc} 
\usepackage[T1]{fontenc}    
\usepackage{url}            
\usepackage{booktabs}       
\usepackage{nicefrac}       

\usepackage{microtype}  
\usepackage{booktabs} 

\usepackage{theoremref}

\usepackage[accepted,nohyperref]{icml2020}

\icmltitlerunning{Optimizing for the Future in  Non-Stationary MDPs}

\begin{document}

\twocolumn[
\icmltitle{Optimizing for the Future in Non-Stationary MDPs}



\icmlsetsymbol{equal}{*}

\begin{icmlauthorlist}
\icmlauthor{Yash Chandak}{umass}
\icmlauthor{Georgios Theocharous}{adobe}
\icmlauthor{Shiv Shankar}{umass}\\
\icmlauthor{Martha White}{uoa}
\icmlauthor{Sridhar Mahadevan}{umass,adobe}
\icmlauthor{Philip S. Thomas}{umass}
\end{icmlauthorlist}

\icmlaffiliation{umass}{University of Massachusetts, MA, USA.}
\icmlaffiliation{uoa}{University of Alberta, AB, Canada}
\icmlaffiliation{adobe}{Adobe Research, CA, USA.}

\icmlcorrespondingauthor{Yash Chandak}{ychandak@cs.umass.edu}


\icmlkeywords{Machine Learning, ICML}

\vskip 0.3in
]




\printAffiliationsAndNotice{}  

\begin{abstract}

Most reinforcement learning methods are based upon the key assumption that the transition dynamics and reward functions are fixed, that is, the underlying Markov decision process is stationary. 
However, in many real-world applications, this assumption is violated and using existing algorithms may result in a performance lag.
To proactively search for a good \textit{future} policy, we present a policy gradient algorithm that maximizes a forecast of \textit{future} performance.
This forecast is obtained by fitting a curve to the counter-factual estimates of policy performance over time, without explicitly modeling the underlying non-stationarity.
The resulting algorithm amounts to a non-uniform reweighting of past data, and we observe that  \textit{minimizing} performance over some of the data from past episodes can be beneficial when searching for a policy that \textit{maximizes} future performance. 
We show that our algorithm,  called Prognosticator, is more robust to non-stationarity than two online adaptation techniques, on three simulated problems motivated by real-world applications.
\end{abstract}
\section{Introduction}
Policy optimization algorithms in RL are promising for obtaining general purpose control algorithms.
They are designed for Markov decision processes (MDPs), which model a large class of problems \cite{sutton2018reinforcement}.
This generality can facilitate application to a variety of real-world problems. 
However, most existing algorithms assume that the transition dynamics and reward functions are fixed, that is, the underlying Markov decision process is stationary.

This assumption is often violated in practical problems of interest.
For example, consider an assistive driving system. 
Over time, tires suffer from wear and tear, leading to increased friction.
Similarly, in almost all human-computer interaction applications, e.g., automated medical care, dialogue systems, and marketing, human behavior changes over time.
In such scenarios, if the automated system is not adapted to take such changes into account, or if it is adapted only after observing such changes, then the system might quickly become sub-optimal, incurring severe loss \citep{moore2014reinforcement}.
This raises our main question: \textit{how do we build systems that proactively search for a policy that will be good for the future MDP?}

%
In this paper we present a policy gradient based approach to search for a policy that maximizes the forecasted future performance when the environment is non-stationary. 
To capture the impact of changes in the environment on a policy's performance, first, the performance of the policy during the past episodes is estimated using counter-factual reasoning.
%
%
Subsequently, a regression curve is fit to these estimates to model the performance trend of the policy over time, thereby enabling the forecast of future performance.
%
%
By differentiating this performance forecast with respect to the parameters of the policy being evaluated, we obtain a gradient-based optimization procedure that proactively searches for a policy that will perform well in the future.\footnote{Code for our algorithm can be obtained using the following link: \href{https://github.com/yashchandak/OptFuture_NSMDP}{https://github.com/yashchandak/OptFuture\_NSMDP}.}
%
%
 %
 
 Recently, \citet{al2017continuous} and \citet{finn2019online} also presented methods that search for  initial parameters that are effective when the objective is changing over time. 
These approaches are complementary to our own, as they could be additionally applied to set the initial parameters of our algorithms. 
In our empirical study, we show how the adaptation procedure of their methods alone can result in a performance lag that is mitigated by our method, which explicitly captures the trend of the objective due to non-stationarity.
%
%
A detailed survey on other approaches can be found in the work by \citet{padakandla2020survey}.

\textbf{Advantages:}  The proposed method has the following advantages: (a)
%
    %
    It does not require modeling the transition function, reward function, or how either changes in a non-stationary environment, and thus scales gracefully with respect to the number of states and actions in the environment.
    %
    %
    (b) Irrespective of the complexity of the environment or the policy parameterization, it concisely models the \textit{effect} of changes in the environment on a policy's performance using a \textit{univariate} time-series. %
    (c) It is data-efficient in that it leverages all available data.
    %
    (d) It mitigates performance lag by proactively optimizing performance for episodes in both the immediate and near future.
    %
    (e) It degenerates to an estimator of the ordinary policy gradient if the system is stationary, meaning that there is little reason not to use our approach if there is a possibility that the system \emph{might} be non-stationary.
%

\textbf{Limitations:}
The method that we propose is limited to settings where  (a) non-stationarity is governed by an exogenous process (i.e., past actions do not impact the underlying non-stationarity), which has no auto-correlated noise, and (b) performance of every policy changes smoothly over time and has no abrupt breaks/jumps.
Further, we found that our method is sensitive to a hyper-parameter that trades off exploration and exploitation in the non-stationary setting.
We conclude the paper with a discussion of these limitations.

\section{Notation}
\topic{Setting up all the notations.}
An MDP $\mathcal M$ is a tuple $(\mathcal S, \mathcal A, \mathcal P, \mathcal R, \gamma, d^0)$, where
$\mathcal S$ is the set of possible states, $\mathcal A$ is the set of actions, $\mathcal P$ is the transition function, $\mathcal R$ is the reward function, $\gamma$ is the discount factor, and $d^0$ is the start state distribution.
Let $\mathcal R(s,a)$ denote the expected reward of taking action $a$ in state $s$.
For any given set $\mathcal X$, we use $\Delta(\mathcal X)$ to denote the set of distributions over $\mathcal X$.
A policy $\pi: \mathcal S \rightarrow \Delta(\mathcal A)$ is a distribution over the actions conditioned on the state.
When $\pi$ is parameterized using $\theta \in \mathbb{R}^d$, we denote it as $\pi^\theta$.
Often we write $\theta$ in place of $\pi^\theta$ when the dependence on $\theta$ is important.
In a non-stationary setting, as the MDP changes over time, we use $\mathcal M_k$ to denote the MDP during episode $k$.
In general, we will use sub-script $k$ to denote the episode number and a super-script $t$ to denote the time-step within an episode.
$S^t_k, A^t_k, $ and $R^t_k$ are the random variables corresponding to the state, action, and reward at time step $t$, in episode $k$. 
Let $H_k$ denote a trajectory in episode $k$: $(s^0_k, a^0_k, r^0_k, s^1_k, a^1_k, ..., s^T_k)$, where $T$ is the finite horizon.
The value function evaluated at state $s$, during episode $k$, under a policy $\pi$ is $v_k^{\pi}(s) = \mathbb{E}[\sum_{j=0}^{T-t} \gamma^j R_k^{t+j}|S_k^t=s, \pi]$, where conditioning on $\pi$ denotes that the trajectory in episode $k$ was sampled using $\pi$.
The start state objective for a policy $\pi$, in episode $k$, is defined as $J_k(\pi) \coloneqq \sum_s d_0(s)v_k^\pi(s)$.
Let $J^*_k = \text{max}_\pi\,\, J_k(\pi)$ be the maximum performance for $\mathcal M_k$. 
\section{Problem Statement}
\label{problemSetup}
%
%
To model non-stationarity, 
we let an exogenous process change the MDP from $\mathcal M_{k}$ to $\mathcal M_{k+1}$, i.e., between episodes.
Let $\{\mathcal M_k\}_{k=1}^\infty$ represent a sequence of MDPs, where each MDP $\mathcal M_k$ is denoted by the tuple  $(\mathcal S, \mathcal A, \mathcal P_k, \mathcal R_k, \gamma, d^0)$.\footnote{An alternate way to formulate this problem could be to convert a non-stationary MDP into a stationary (PO)MDP by considering an `expanded' (PO)MDP consisting of all possible variants of the given problem.
Then a single, never-ending, episode in that (PO)MDP can be considered with state dependent discounting \cite{white2017unifying} for task specifications.
We plan to explore the single, never-ending, episode formulation in future work.
} 

In many problems, like adapting to friction in robotics, human-machine interaction, etc., the transition dynamics and rewards function change, but every other aspect of the MDP remains the same throughout.
Therefore, we assume that 
%
%
for any two MDPs, $\mathcal M_k$ and $\mathcal M_{k+1}$, the state set $\mathcal S$, the action set $\mathcal A$, the starting distribution $d^0$, and the discount factor $\gamma$ are the same.
%
%

If the exogenous process changing the MDPs is arbitrary and changes it in unreasonable ways, then there is little hope of finding a good policy for the future MDP as $\mathcal M_{k+1}$ can be wildly different from everything the agent has observed by interacting with the past MDPs, $\mathcal M_{1},..., \mathcal M_{k}$. 
However, in many practical problems of interest, such changes are smooth and have an underlying (unknown) structure.
To make the problem tractable, we therefore assume that both the transition dynamics $(\mathcal P_1,\mathcal P_2, ...)$, and the reward functions $( \mathcal R_1, \mathcal R_2,... )$ vary smoothly over time in a way that ensures there are no abrupt jumps in the performance of any policy.  
%
%

\textbf{Problem Statement. } 
We seek to find a sequence of policies that minimizes lifelong regret: 
$$\argmin_{\{\pi_1, \pi_2, ... \pi_k, ...\}} \, \sum_{k=1}^\infty J_k^* - \sum_{k=1}^\infty J_k(\pi_k).
$$

\section{Related Work}
%

%
The problem of non-stationarity has a long history and no effort is enough to thoroughly review it. 
Here, we briefly touch upon the most relevant work and defer a more detailed literature review to the appendix. 
A more exhaustive survey can be found in the work by \citet{padakandla2020survey}.

Perhaps the work most closely related to ours is that of \citet{al2017continuous}. 
They consider a setting where an agent is required to solve test tasks that have different transition dynamics than the training tasks.
Using meta-learning, they aim to use training tasks to find an initialization vector for the policy parameters that can be quickly fine-tuned when facing tasks in the test set.
In many real-world problems, however, access to such independent training tasks may not be available \textit{a priori}.
In this work, we are interested in the continually changing setting where there is no boundary between training and testing tasks.
As such, we show how their proposed online adaptation technique that fine-tunes parameters, by discarding past data and only using samples observed online, can create performance lag and can therefore be data-inefficient. 
In settings where training and testing tasks do exist, our method can be leveraged to better adapt during test time, starting from any desired parameter vector.

%
Recent work by \citet{finn2019online} aims at bridging both the continuously changing setting and the train-test setting for supervised-learning problems.
They propose continuously improving an underlying parameter initialization vector and running a Follow-The-Leader (FTL) algorithm \citep{shalev2012online} 
every time new data is observed.
A naive adaption of this for RL would require access to all the underlying MDPs in the past for continuously updating the initialization vector, which would be impractical.
Doing this efficiently remains an open question and our method is complementary to choosing the initialization vector. 
Additionally, FTL based adaptation  always lags in tracking optimal performance as it uniformly maximizes performance over all the past samples that might not be directly related to the future.
Further, we show that by explicitly capturing the trend in the non-stationarity, we can mitigate this performance lag resulting from the use of an FTL algorithm during the adaptation process.

The problem of adapting to non-stationarity is also related to continual learning \citep{ring1994continual}, lifelong-learning \citep{thrun1998lifelong}, and meta-learning \citep{schmidhuber1999general}.
Several meta-learning based approaches for fine-tuning a (mixture of) trained model(s) using 
samples observed during a similar task at test time have been proposed \citep{nagabandi2018learning,nagabandi2018deep}.
%
%
Other works have also shown how models of the environment can be used for continual learning \citep{lu2019adaptive} or using it along with a model predictive control \citep{wagener2019online}.
Concurrent work by \citet{xie2020deep} also demonstrates how modeling the changes in a \textit{dynamic-parameter} MDP can be useful.
We focus on the model-free paradigm and our approach is complementary to these model-based methods.

More importantly, in many real-world applications, it can be infeasible to update the system frequently if it involves high computational or monetary expense.
In such a case, even optimizing for the immediate future might be greedy and sub-optimal.
The system should optimize for a longer term in the future, to compensate for the time until the next update is performed.
None of the prior approaches can efficiently tackle this problem.

\section{Optimizing for the Future}

%
The problem of minimizing lifelong regret is straightforward if the agent has access to sufficient samples, in advance, from the future environment, $\mathcal M_{k+1}$, that it is going to face (where $k$ denotes the current episode number).
That is, if we could estimate the start-state objective, $J_{k+1}(\pi)$, for the future MDP $\mathcal M_{k+1}$, then we could search for a policy $\pi$ whose performance is close to $J_{K+1}^*$. 
However, obtaining even a single sample from the future is impossible, let alone getting a sufficient number of samples.
This necessitates rethinking the optimization paradigm for searching for a policy 
that performs well when faced with the future unknown MDP.
There are two immediate challenges here:
\begin{enumerate}
    \item \textit{
    How can we estimate $J_{k+1}(\pi)$ without any samples from $\mathcal M_{k+1}$?}
    \item \textit{How can gradients, $\partial J_{k+1}(\pi)/ \partial \theta$, of this future performance be estimated?}
\end{enumerate}
%
In this section we address both of these issues using the following idea.
When the transition dynamics $(\mathcal P_1,\mathcal P_2, ...)$, and the reward functions $( \mathcal R_1, \mathcal R_2,... )$ are changing smoothly the performances $(J_1(\pi),J_2(\pi), ...)$ of any policy $\pi$ can also be expected to vary smoothly over time.
The impact of smooth changes in the environment thus manifests as smooth changes in the performance of any policy, $\pi$.
In cases where there is an underlying, unknown, structure in the changes of the environment, one can now ask: \textit{if the performances $J_{1:k}(\pi) \coloneqq (J_1(\pi), ..., J_{k}(\pi))$ of $\pi$ over the course of past episodes were known, can we analyze the trend in its past performances to find a policy that maximizes future performance $J_{k+1}(\pi)$?}

\subsection{Forecasting Future Performance}
 %
 In this section we address the first challenge of estimating future performance $J_{k+1}(\pi)$ and pose it as a time series forecasting problem.

Broadly, this requires two components: (a) A procedure to compute past performances, $J_{1:k}(\pi)$, of $\pi$. (b) A procedure to create an estimate, $\hat J_{k+1}(\pi)$, of $\pi$'s future performance, $J_{k+1}(\pi)$, using these estimated values from component (a).
 An illustration of this idea is provided in Figure \ref{fig:eval}.

\begin{figure}
    \centering
    \includegraphics[width=0.35\textwidth]{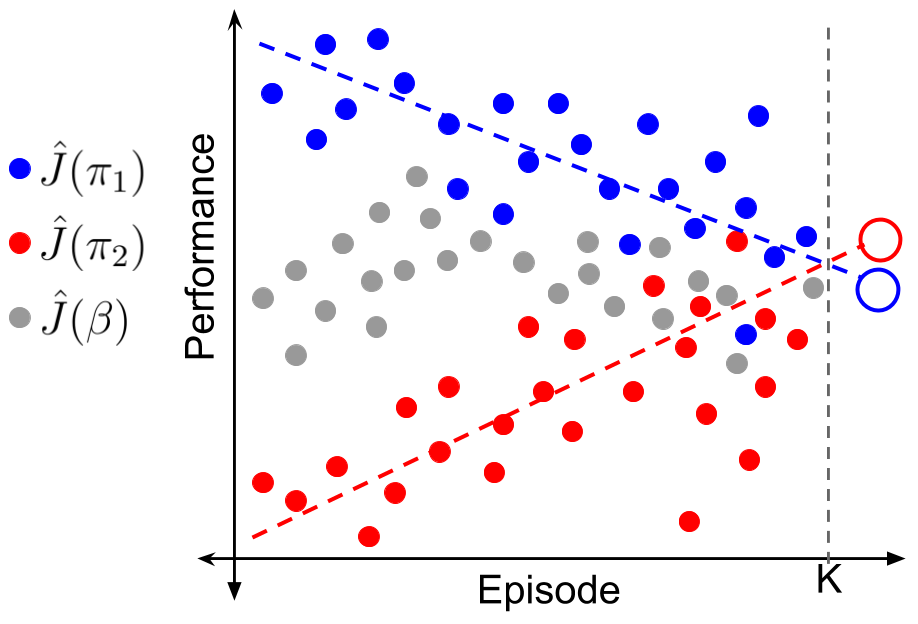}
    \caption{
    An illustration, where the blue and red filled circles represent counter-factually reasoned performance estimates of policies $\pi_1$ and $\pi_2$, respectively, using data collected from a given policy $\beta$. 
    The open circles represent the forecasted performance of $\pi_1$ and $\pi_2$ estimated by fitting a curve on the counter-factual estimates represented by filled circles. }
    \vspace{-5pt}
    \label{fig:eval}
\end{figure}

\textbf{Component (a).}  
As we do not have access to the past MDPs for computing the true values of past performances, $J_{1:k}(\pi)$, we propose computing estimates, $\hat J_{1:k}(\pi)$, of them from the observed data.
That is, in a non-stationary MDP, starting with the fixed transition matrix $\mathcal P_1$ and the reward function $\mathcal R_1$, we want to estimate the performance $J_{i}(\pi)$ of a given policy in episode $i \leq k$.  
Leveraging the fact that the changes to the underlying MDP are due to an exogenous processes, we can estimate $J_i(\pi)$ as,
\begin{align} J_i(\pi) =   \sum_{t=0}^{T}\gamma^t\mathbb{E}\left[ R_i^t  \middle | \pi, \mathcal P_i, \mathcal R_i\right], \label{eqn:fromcurrent}
\end{align}
where $\mathcal P_i$ and $\mathcal R_i$ are also random variables.
Next we describe how an estimate of $J_i(\pi)$ can be obtained from \eqref{eqn:fromcurrent} using information only from the $i^\text{th}$ episode.

%
To get an unbiased estimate, $\hat J_{i}(\pi)$, of $\pi$'s performance during episode $i$, consider the past trajectory $H_i$ of the $i^\text{th}$ episode that was observed when executing a policy $\beta_i$.
By using counter-factual reasoning \citep{rosenbaum1983central} 
and leveraging the per-decision importance sampling (PDIS) estimator \citep{precup2000eligibility}, an unbiased estimate of $J_i(\pi)$ is thus given by:\footnote{We assume that $\forall i \in \mathbb{N}$ the distribution of $H_i$ has full support over the set of all possible trajectories of the MDP $\mathcal M_i$.}
\begin{align}
     \hat J_i(\pi) &:=  \sum_{t=0}^{H} \left ( \prod_{l=0}^t 
    \frac{\pi(A_i^l|S_i^l)}{\beta_i(A_i^l|S_i^l)} \right) \gamma^t R_i^t. \label{eqn:ope}
\end{align}

It is worth noting that computing \eqref{eqn:ope} does not require storing all the past policies $\beta_i$, one needs to only store the actions and the probabilities with which these actions were chosen. 

%
\textbf{Component (b).} 
To obtain the second component, which captures the structure in $\hat J_{1:k}(\pi) \coloneqq (\hat J_1(\pi), ..., \hat J_{k}(\pi)) $ and predicts future performances, we make use of a forecasting function $\Psi$ that estimates future performance $\hat J_{k+1}(\pi)$ conditioned on the past performances:
\begin{align}
    \hat J_{k+1}(\theta) \coloneqq \Psi (\hat J_1(\pi), \hat J_2(\pi), ...., \hat J_k(\pi)). \label{eqn:forecaster}
\end{align}
%

%
While $\Psi$ can be any forecasting function, we consider $\Psi$ to be an ordinary least squares (OLS) regression model with parameters $w \in \mathbb{R}^{d\times 1}$, and the following input and output variables, 
\begin{align}
    X &\coloneqq [1, 2, ..., k]^\top \quad\quad  & \in \mathbb{R}^{k \times 1}, \\
    Y &\coloneqq [\hat J_1(\pi), \hat J_2(\pi), \hat J_2(\pi), ..., \hat J_k(\pi)]^\top & \in \mathbb{R}^{k \times 1}. 
\end{align}
For any $x \in X$, let $\phi(x) \in \mathbb{R}^{1 \times d}$ denote a $d$-dimensional basis function for encoding the time index.
For example, an identity basis $\phi(x) \coloneqq \{x, 1 \}$, or a Fourier basis, where
$$\phi(x) \coloneqq \{\sin(2 \pi n x | n \in \mathbb{N}\} \cup \{\cos(2 \pi n x) | n \in \mathbb{N} \} \cup \{1\}. $$
Let $\Phi \in \mathbb{R}^{k\times d}$ be the corresponding basis matrix.
The solution to above least squares problem is $ w = (\Phi^\top\Phi)^{-1}\Phi^\top Y$ \citep{bishop2006pattern} and the forecast of the future performance can be obtained using,
\begin{align}
   \hat J_{k+1}(\pi) =  \phi(k+1) w  = \phi(k+1)(\Phi^\top  \Phi)^{-1}\Phi^\top Y. \label{eqn:predict}
\end{align}

\topic{Advantages of the proposed approach for estimating future performance.}
This procedure enjoys an important advantage -- by just using a univariate time-series to estimate future performance, it bypasses the need for modeling the environment, which can be prohibitively hard or even impossible.
Further, note that $\Phi ^\top \Phi \in \mathbb{R}^{d \times d}$, where $d << k$ typically, and thus the cost of computing the matrix inverse is negligible.
These advantages allows this procedure to gracefully scale to more challenging problems, while being robust to the size, $|\mathcal S|$, of the state set or the action set $|\mathcal A|$.

\subsection{Differentiating Forecasted Future Performance}

%
In the previous section, we addressed the first challenge and showed how to proactively estimate future performance, $\hat J_{k+1}(\theta)$, of a policy $\pi^\theta$ by explicitly modeling the trend in its past performances $\hat J_{1:k}(\theta)$.
In this section, we address the second challenge to facilitate a complete optimization procedure.
 A pictorial illustration of the idea is provided in Figure \ref{fig:idea}.

Gradients for $\hat J_{k+1}(\theta)$ with respect to $\theta$ can be obtained as follows,
\begin{align}
    \frac{ d \hat J_{k+1}(\theta)}{d \theta} &= \frac{d \Psi(\hat J_1(\theta), ..., \hat J_k(\theta))}{d \theta} \\
    &= \sum_{i=1}^k \underbrace{ \frac{\partial \Psi(\hat J_1(\theta), ..., \hat J_k(\theta))}{\partial \hat J_i(\theta)}}_{\text{(a)}}\cdot \underbrace{ \frac{d \hat J_i(\theta)}{d \theta}}_{(b)}. \,\,\, \label{eqn:future-grad}
\end{align}
%

The decomposition in \eqref{eqn:future-grad} has an elegant intuitive interpretation.
The terms assigned to $(a)$ in \eqref{eqn:future-grad} correspond to how the future prediction would change as a function of past outcomes, and
the terms in $(b)$ indicate how the past outcomes would change due to changes in the parameters of the policy $\pi^\theta$. 
In the next paragraphs, we discuss how to obtain the terms $(a)$ and $(b)$.
%
%

To obtain term (a), note that in \eqref{eqn:predict}, $\hat J_i(\theta)$ corresponds to the $i^\text{th}$ element of $Y$, and so using \eqref{eqn:forecaster} the gradients of the terms (a) in \eqref{eqn:future-grad} are,
\begin{align}
    \frac{\partial \hat J_{k+1}(\theta)}{\partial \hat J_i(\theta)}  &= \frac{\partial \phi(k+1)(\Phi^\top  \Phi)^{-1}\Phi^\top  Y}{\partial Y_i} \\
    &= [ \phi(k+1)(\Phi^\top  \Phi)^{-1}\Phi^\top  ]_i, \label{eqn:extra-grad}
\end{align}
where $[Z]_i$ represents the $i^\text{th}$ element of a vector $Z$.
Therefore, \eqref{eqn:extra-grad} is the \textit{gradient of predicted future performance with respect to an estimated past performance}.
%

The term $(b)$ in \eqref{eqn:future-grad} corresponds to the gradient of the PDIS estimate $\hat J_i(\theta)$ of the past performance with respect to policy parameters $\theta$.
The following Property provides a form for (b) that makes its computation straightforward. 
\begin{prop}[PDIS gradient]
Let $\rho_i(0,l) \coloneqq \prod_{j=0}^{l} \frac{\pi^\theta(A_i^j|S_i^j)}{\beta_i(A_i^j|S_i^j)}$.
\begin{align}
    \frac{d \hat J_i(\theta)}{d\theta} &= \sum_{t=0}^{T}  \frac{  \partial \log  \pi^\theta(A_i^t|S_i^t)}{\partial \theta}  \left( \sum_{l=t}^{T} \rho_i(0, l) \gamma ^{l}  R_i^l \right). \label{eqn:opg3}
\end{align}
\end{prop}

\begin{proof}
    See Appendix \ref{apx:sec:gradients}. 
\end{proof}

\begin{figure}
    \centering
    \includegraphics[width=0.45\textwidth]{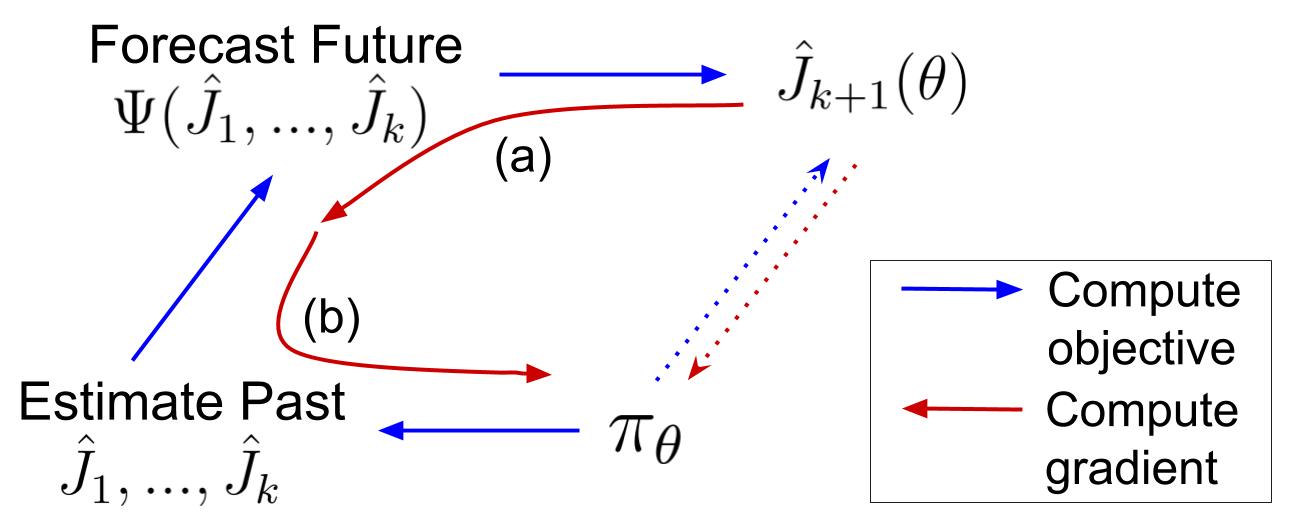}
    \caption{
    The proposed method from the lens of differentiable programming.
    At any time $k$, we aim to optimize policy's parameters, $\theta$, to maximize its performance in the future, $J_{k+1}(\theta)$.
    However, conventional methods (dotted arrows) can not be used to directly optimize for this.
    In this work, we achieve this as a composition of two programs: one which connects the policy's parameters to its past performances, and the other which forecasts future performance as a function of these past performances.
    The optimization procedure then corresponds to taking derivatives through this composition of programs to update policy parameters in a direction that maximizes future performance.
    %
    %
    %
    Arrows (a) and (b) correspond to the respective terms marked in \eqref{eqn:future-grad}.}
    \label{fig:idea}
\end{figure}

%

\subsection{Algorithm}

We provide a sketch of our proposed \textit{Prognosticator} procedure for optimizing the future performance of the policy in Algorithm \ref{apx:Alg:1}.
To make the method more practical, we incorporated two additional modifications to reduce computational cost and variance.
First, it is often desirable to perform an update only after a certain episode interval $\delta$ to reduce computational cost. 
This raises the question: if a newly found policy will be executed for the next $\delta$ episodes, should we choose this new policy to maximize performance on just the single next episode, or to maximize the average performance over the next $\delta$ episodes? 
An advantage of our proposed method is that we can easily tune how far in the future we want to optimize for. Thus,
to minimize lifelong regret, we propose optimizing for the mean performance over the next $\delta$ episodes.
That is, $\text{arg\,max}_\theta \, (1/\delta)\sum_{\Delta=1}^\delta \hat J_{k+\Delta}(\theta)$.
    %

Second, notice that if the policy becomes too close to deterministic, there would be two undesired consequences. 
(a) The policy will not cause exploration, precluding the agent from observing any changes to the environment in states that it no longer revisits---changes that might make entering those states worthwhile to the agent. 
(b) In the future when estimating $\hat J_{k+1}(\theta)$ using the \emph{past} performance of $\theta$, importance sampling will have high variance if the policy executed during episode $k+1$ is close to deterministic. 
To mitigate these issues, we add an entropy regularizer $\mathcal H$ during policy optimization.
More details are available in Appendix \ref{apx:sec:implementation}.

	\IncMargin{1em}
	\begin{algorithm2e}[t]
		\textbf{Input} Learning-rate $\eta$, time-duration $\delta$, entropy-regularizer  $\lambda$\\
		\textbf{Initialize} Forecasting function $\Psi$, Buffer  $\mathds{B}$\\
		\While{True}{
			\nonl \textcolor[rgb]{0.5,0.5,0.5}{\# Record a new batch of trajectories using $\pi^\theta$}
			\\
			\For {$episode = 1,2,..., \delta$}{
    		$h = \{(s_{0:T}, a_{0:T}, \Pr(a_{0:T}|s_{0:T}), r_{0:T}) \}$  
    		\\
    		$\mathds{B}.\text{insert}(h)$
    		}
    		
    		\vspace{8pt}
    		\nonl \textcolor[rgb]{0.5,0.5,0.5}{\# Update for future performance}
    		\\
    		\For{$i = 1, 2, ...$}{
    			\nonl \textcolor[rgb]{0.5,0.5,0.5}{\# Evaluate past performances}\\
    		    \For {$k =  1, 2, ..., |\mathds{B}|$ }{
    		        $\hat J_k(\theta) =  \sum_{t=0}^{T} \rho (0, t) \gamma ^{t}  R^t_k $ \Comment{\eqref{eqn:ope}}
    		    }
    		    \vspace{5pt}
    		    %
    		    \nonl \textcolor[rgb]{0.5,0.5,0.5}{\# Future forecast and its gradient}\\
    		    $\mathcal L(\theta) = \frac{1}{\delta}\sum_{\Delta=1}^\delta \hat J_{k+\Delta}(\theta)$ \Comment{\eqref{eqn:predict}}
    		    \\
    		    $ \theta \leftarrow \theta + \eta \frac{\partial}{\partial \theta} (\mathcal L(\theta) + \lambda \mathcal H(\theta))$  \Comment{\eqref{eqn:future-grad}}
    		}
		}  
		\caption{Prognosticator}
		\label{apx:Alg:1}  
	\end{algorithm2e}
	\DecMargin{1em} 

\section{Understanding the Behavior of Prognosticator }

Notice that as the scalar term (a) is multiplied by the PDIS gradient term (b) in \eqref{eqn:future-grad}, the gradient of future performance can be viewed as a weighted sum of off-policy policy gradients.
In Figure \ref{fig:weights}, we provide visualization of the weights $\zeta_i \coloneqq \partial \hat J_{100}(\theta)/ \partial \hat J_i(\theta)$ for PDIS gradients of each episode $i$, when the performance for $100^\text{th}$ episode is forecasted using data from the past $99$ episodes.
For the specific setting when $\Psi$ is an OLS estimator, these  weights  are independent of $Y$ in \eqref{eqn:extra-grad} and their pattern remains constant for any given sequence of MDPs.
Importantly, note the occurrence of negative weights in Figure \ref{fig:weights} when the identity basis 
or Fourier basis is used, suggesting that the optimization procedure should move towards a policy that had \textit{lower} performance in some of the past episodes.
While this negative weighting seems unusual at first glance, it has an intriguing interpretation.
%

\begin{figure}
    \centering
    \includegraphics[width=0.4\textwidth]{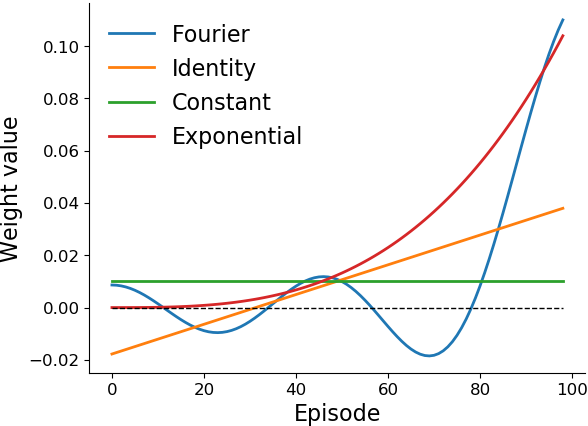}
    \caption{
    The value of weights $\zeta_i$ for all values of $i \in [1, 99]$ using different basis functions to encode the time index. 
    Notice that many weights are negative when using the identity or Fourier bases.}
    \label{fig:weights}
\end{figure}

To better understand these negative weights, 
consider a qualitative comparison when weights from different methods in Figure \ref{fig:weights} are used along with the performance estimates of policies $\pi_1$ and $\pi_2$ in Figure \ref{fig:eval}.
Despite having lower estimates of return everywhere, $\pi_2$'s rising trend suggests that it might have higher performance in the future, that is, $J_{k+1}(\pi_2) > J_{k+1}(\pi_1)$.
Existing online learning methods like FTL, maximize performance on all the past data uniformly (green curve in Figure \ref{fig:weights}). 
Similarly, the exponential weights (red curve in Figure \ref{fig:weights}) are representative of approaches that only optimize using data from recent episodes and discard previous data.
Either of these methods that use only non-negative weights can never capture the trend to forecast $J_{k+1}(\pi_2) > J_{k+1}(\pi_1)$. 
However, the weights obtained when using the identity basis would facilitate \textit{minimization} of performances in the distant past and maximization of performance in the recent past.
Intuitively, this means that it moves towards a policy whose performance is on a linear rise, as it expects that policy to have better performance in the future.
%

While weights from the identity basis are useful for forecasting whether $J_{k+1}(\pi_2) > J_{k+1}(\pi_1)$, it cannot be expected that the trend will always be linear as in Figure \ref{fig:eval}. 
To be more flexible and allow for any smooth trend, we opt to use the Fourier basis in our experiments. 
Observe the alternating sign of weights in Figure \ref{fig:weights} when using the Fourier basis. 
This suggests that the optimization procedure will take into account the \textit{sequential differences} in performances over the past, thereby favoring the policy that has shown the most performance increments in the past.
This also avoids restricting the performance trend of a policy to be linear.

\section{Mitigating Variance}
%
%
While model-free algorithms for finding a good policy are scalable to large problems, they tend to suffer from high-variance \citep{greensmith2004variance}.
In particular, the use of importance sampling estimators can increase the variance further \citep{guo2017using}.
In our setup, high variance in estimates of past performances $\hat J_{1:k}(\pi)$ of $\pi$ can hinder capturing $\pi$'s performance trend, thereby making the forecasts less reliable.
%

Notice that a major source of variance is the availability of only a \textit{single} trajectory sample per MDP $\mathcal M_i$, for all $i \in \mathbb{N}$.
If this trajectory $H_i$, generated using $\beta_i$ is likely when using $\beta_i$, but has near-zero probability when using $\pi$ then the estimated $\hat J_i(\pi)$ is also nearly zero.
While $\hat J_i(\pi)$ is an unbiased estimate of $ J_i(\pi)$, information provided by this single $H_i$ is of little use to evaluate $J_i(\pi)$. 
Subsequently, discarding this from time-series analysis, rather than setting it to be $0$, can make the time series forecast more robust against outliers.
In comparison, if trajectory $H_i$ is unlikely when using $\beta_i$ but likely when using $\pi$, then not only is $H_i$ very useful for estimating $J_i(\pi)$ but  it also has a lower chance of occurring in the future, so this trajectory must be emphasized when making a forecast.
Such a process of (de-)emphasizing estimates of past returns using the collected data itself can introduce bias, but this bias might be beneficial in this few-sample regime.

To capture this idea formally, we build upon the insights of \citet{hachiya2012importance} and \citet{mahmood2014weighted}, who draw an equivalence between weighted least-squares (WLS) estimation and the weighted importance sampling (WIS)  \citep{precup2000eligibility} estimator.
%
Particularly, let $G_i := \sum_{t=0}^T \gamma^t R^t_i$ be the discounted return of the $i^\text{th}$ trajectory observed from a stationary MDP, and $\rho^\ddagger_i \coloneqq \rho_i(0, T)$ be the importance ratio of the entire trajectory.
Then the WIS estimator, $\hat J^\ddagger (\pi)$, of the performance of $\pi$ in a stationary MDP is,
\begin{align}
  \hat J^\ddagger (\pi) &:= \underset{c \in \mathbb{R}}{\text{argmin}} \frac{1}{n}\sum_{i=1}^n \rho_i^\ddagger(G_i - c)^2 = \frac{\sum_{i=1}^n \rho_i^\ddagger G_i}{\sum_{i=1}^n \rho_i^\ddagger}. \label{eqn:WLS_rederived}
\end{align}
%

To mitigate variance in our setup, we propose extending WIS.
In the non-stationary setting, to perform WIS while capturing the trend in performance over time, we use a modified  forecasting function $\Psi^\ddagger$, which is a weighted least-squares regression model with a $d-$dimensional basis function $\phi$, and parameters $w^\ddagger \in \mathbb{R}^{d\times 1}$,
\begin{align}
  w^\ddagger &:= \underset{c \in \mathbb{R}^{d \times 1}}{\text{argmin}} \frac{1}{n}\sum_{i=1}^n \rho_i^\ddagger(G_i - c^\top \phi(i))^2. \label{eqn:weightedLS} 
\end{align}
Let  $\Lambda \in \mathbb{R}^{k \times k}$ be a diagonal weight matrix such that $\Lambda_{ii} = \rho_i^\ddagger$, let $\Phi \in \mathbb{R}^{k\times d}$ be the basis matrix, and let the following be input and output variables, 
\begin{align}
    X &\coloneqq [1, 2, ..., k]^\top \quad\quad  & \in \mathbb{R}^{k \times 1}, \\
    Y &\coloneqq  [G_1, G_2, ..., G_k]^\top & \in \mathbb{R}^{k \times 1}.
\end{align}
The solution to the weighted least squares problem in \eqref{eqn:weightedLS} is then given by $ w^\ddagger = (\Phi^\top \Lambda \Phi)^{-1}\Phi^\top \Lambda Y$ and the forecast of the future performance can be obtained using,
   $$\hat J^\ddagger _{k+1}(\pi) :=  \phi(k+1) w^\ddagger  = \phi(k+1)(\Phi^\top \Lambda \Phi)^{-1}\Phi^\top \Lambda Y. \label{eqn:predictweighted}$$
%

 $\hat J^\ddagger _{k+1}(\pi)$ has several desired properties.
It incorporates a notion of how relevant each observed trajectory is towards forecasting, while also capturing the trend in performance.
The forecasts are less sensitive to the importance sampling variances and the entire process is still differentiable.
%

\section{Generalizing to the Stationary Setting}
%
As the agent is unaware of how the environment is changing, a natural question to ask is: What if the agent wrongly assumed a stationary environment was non-stationary?
What is the quality of of the agent's performance forecasts?
What is the impact of the negative weights on past evaluations of a policy's performance?
%
Here we answer these questions. 
%

Before stating the formal results, we introduce some necessary notation and two additional properties.
%
%
Let $J(\pi)$ be the performance of policy $\pi$ for a stationary MDP.
Let $\hat J_{k+\delta}(\pi)$ and $\hat J^\ddagger _{k+\delta}(\pi)$ be the non-stationary importance sampling (NIS) and non-stationary weighted importance sampling (NWIS) estimators of performance $\delta$ episodes in future.
Further, let the basis function $\phi$ used for encoding the time index in both $\Psi$ and $\Psi^\ddagger$ be such that it satisfies the following conditions: 
(a) $\phi(\cdot)$ always contains $1$ to incorporate a bias/intercept coefficient in least-squares regression (for example, $\phi(\cdot) = [\phi_1(\cdot), ..., \phi_{d-1}(\cdot), 1]$, where $\phi_i(\cdot)$ are arbitrary functions).
(b) $\Phi$ has full column rank such that $(\Phi^\top \Phi)^{-1}$ exists.
Both these properties are trivially satisfied by most basis functions.
With this notation and these properties, we first formalize the stationarity assumption: 
\begin{ass}[Stationarity]
For all $i$, $\mathcal M_i = \mathcal M_{i+1}$. 
\thlabel{ass:stationary}
\end{ass}
This implies that $\mathbb{E}[ \hat J_i(\pi)] = J(\pi)$. 
Following prior literature \cite{precup2000eligibility,thomas2015safe,mahmood2014weighted} we also make a simplifying assumption that allows us to later apply a standard form of the laws of large numbers: %
\begin{ass}[Independence]
$\hat J_i(\pi)$ are independent for all $i$ in $\{1,\dotsc, k\}$.
\thlabel{ass:independentTraj}
\end{ass}
This assumption is satisfied if there is only one behavior policy (i.e., $\forall i, \beta_i=\beta_{i+1}$) or if the sequence of behavior policies does not depend on the data. 
This assumption is not satisfied when the sequence of behavior policies depends on the data because then episodes are not independent. 
While we expect that the following theorems apply even without Assumption \ref{ass:independentTraj}, we have not established this result formally.

We then have the following results indicating that NIS is unbiased and consistent like ordinary importance sampling and NWIS is biased and consistent like weighted importance sampling.
%
%
\begin{figure}
    \centering
    \includegraphics[width=0.4\textwidth]{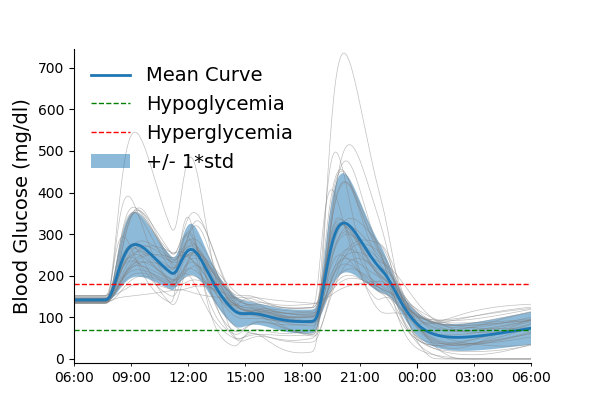}
    \caption{Blood-glucose level of an \textit{in-silico} patient for $24$ hours (one episode). 
     Humps in the graph occur at times when a meal is consumed by the patient. }
    \label{fig:my_label}
\end{figure}

\begin{thm}[Unbiased NIS] Under \thref{ass:stationary,ass:independentTraj}, for all $\delta \geq 1$, $\hat J_{k+\delta}(\pi)$ is an unbiased estimator of $J(\pi)$. That is, $\mathbb{E}[\hat J_{k+\delta}(\pi)] = J(\pi)$.
\thlabel{thm:unbiasedNIS}
\end{thm} 

\begin{thm}[Biased NWIS]
Under \thref{ass:stationary,ass:independentTraj},
for all $\delta \geq 1$, $\hat J^\ddagger _{k+\delta}(\pi)$ may be a \textit{biased} estimator of $J(\pi)$. That is, it is possible that $\mathbb{E}[\hat J^\ddagger _{k+\delta}(\pi)] \neq J(\pi)$.
\thlabel{thm:unbiasedNWIS}
\end{thm} 

\begin{thm}[Consistent NIS] Under \thref{ass:stationary,ass:independentTraj}, for all $\delta \geq 1$, $\hat J_{k+\delta}(\pi)$ is a consistent estimator of $J(\pi)$. That is, as $N \rightarrow \infty,\, \hat J_{N+\delta}(\pi) \overset{\text{a.s.}}{\longrightarrow} J(\pi).$
\end{thm}

\begin{thm}[Consistent NWIS] Under \thref{ass:stationary,ass:independentTraj}, for all $\delta \geq 1$, $\hat J_{k+\delta}^\ddagger(\pi)$ is a consistent estimator of $J(\pi)$. That is, as $N \rightarrow \infty, \, \hat J_{N+\delta}^\ddagger(\pi) \overset{\text{a.s.}}{\longrightarrow} J(\pi).$
\end{thm}

\begin{proof}
    See Appendix \ref{apx:sec:properties} for all of these proofs.
\end{proof}
NWIS is biased and  consistent like the WIS estimator, and our experiments show that it also has similar variance reduction properties that can make the optimization process more efficient for non-stationary MDPs when the variance of $\hat J_i(\pi)$ is high.

\section{Empirical Analysis}

This section presents empirical evaluations using several environments inspired by real-world applications that exhibit non-stationarity.
In the following paragraphs, we briefly discuss each environment; a more detailed description is available in Appendix \ref{apx:sec:implementation}.

\begin{figure*}[t]
    \centering
    \includegraphics[width=0.32\textwidth]{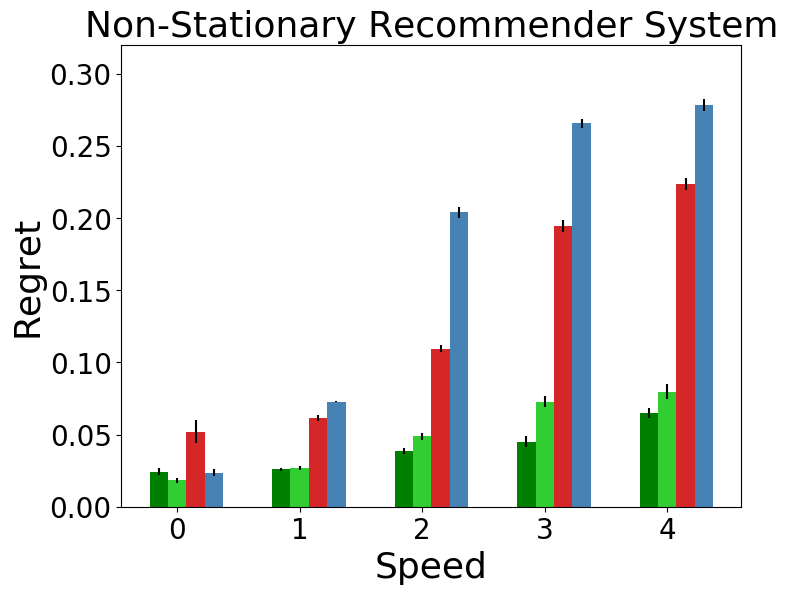}      \includegraphics[width=0.3\textwidth]{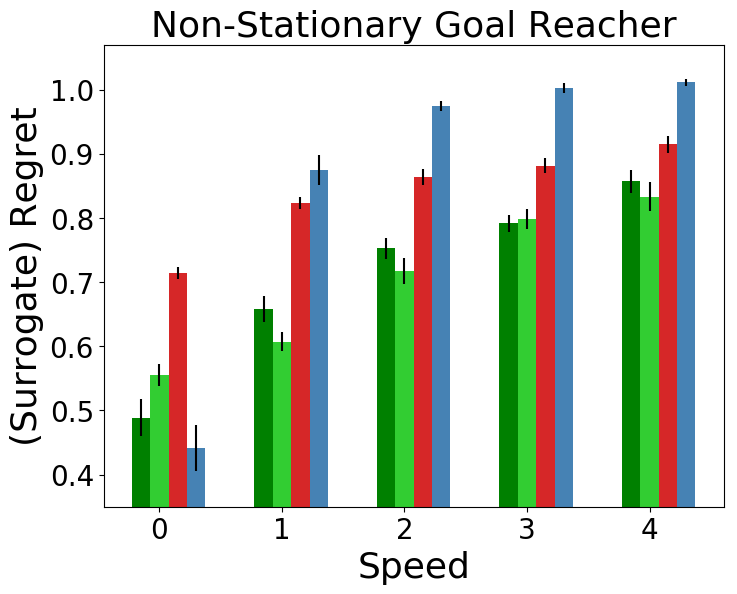}  
    \includegraphics[width=0.3\textwidth]{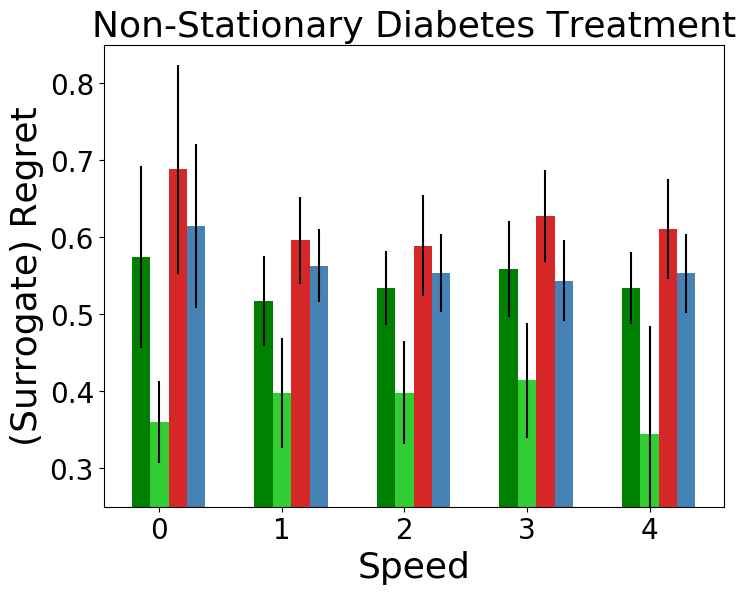} \\
    \includegraphics[width=0.5\textwidth]{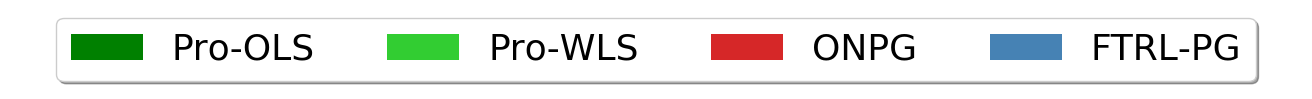} 
    \caption{ 
    Best performances of all the algorithms obtained by conducting a hyper-parameter sweep over $2000$ hyper-parameter combinations per algorithm, per environment.
    For each hyper-parameter setting, $30$ trials were executed for the recommender system and the goal reacher environments, and $10$ trials for the diabetes treatment environment.
    Error bars correspond to the standard error.
    The x-axis represents how fast the environment is changing and the y-axis represents regret (lower is better).
    Individual learning curves for each speed, for each domain, is available in Appendix \ref{apx:sec:results}.
    }
    \label{fig:results}
\end{figure*}

\textbf{Non-stationary Diabetes Treatment:}

This environment is based on an open-source implementation \citep{simglucose} of the FDA approved Type-1 Diabetes Mellitus simulator (T1DMS) \citep{man2014uva} for treatment of Type-1 Diabetes.
Each episode consists of a day in an \textit{in-silico} patient's life.
Consumption of a meal increases the blood-glucose level in the body and if the blood-glucose level becomes too high, then the patient suffers from hyperglycemia and if the level becomes too low, then the patient suffers from hypoglycemia.
The goal is to control the blood-glucose level by regulating the insulin dosage to minimize the risk associated with both hyper and hypoglycemia.

However, the insulin sensitivity of a patient's internal body organs vary over time, inducing non-stationarity that should be accounted for.
In the T1DMS simulator, we induce this non-stationarity by oscillating the body parameters (e.g., insulin sensitivity, rate of glucose absorption, etc.)  between two known configurations available in the simulator. 

\textbf{Non-stationary Recommender System:} In this environment a recommender engine interacts with a user whose interests in different items fluctuate over time.
In particular, the rewards associated with each item vary in seasonal cycles.
The goal is to maximize revenue by recommending an item that the user is most interested in at any time.

\textbf{Non-stationary Goal Reacher:}
This is a 2D environment with four (left, right, up, and down) actions and a continuous state set representing the Cartesian coordinates.
The goal is to make the agent reach a moving goal position.

For all of the above environments, we regulate the \textit{speed} of non-stationarity to characterize an algorithms' ability to adapt.
Higher speed corresponds to a greater amount of non-stationarity;  
A speed of zero indicates that the environment is stationary.

We consider the following algorithms for comparison:

\textbf{Prognosticator:} 
    Two variants of our algorithm, \textbf{Pro-OLS} and \textbf{Pro-WLS}, which use OLS and WLS estimators for $\Psi$. 
    %
    %

\textbf{ONPG:}
Similar to the adaptation technique presented by \citet{al2017continuous}, this baseline performs purely online optimization by fine-tuning the existing policy
using only the trajectory being observed online. 

\textbf{FTRL-PG:} Similar to the adaptation technique presented by \citet{finn2019online}, this baseline performs Follow-the-(regularized)-leader optimization by maximizing performance over both the current and all the past trajectories.

\subsection{Results}

In the non-stationary recommender system, as the exact value of $J_k^*$ is available from the simulator, we can compute the true value of regret.
However, for the non-stationary goal reacher and diabetes treatment environment, as $J_k^*$ is not known for any $k$, we use a surrogate measure for regret. 
That is, let $\tilde J_k^*$ be the maximum return obtained in episode $k$ by any algorithm, then we use $(\sum_{k=1}^N (\tilde J_k^* -  J_k(\pi)))/(\sum_{k=1}^N  \tilde J_k^*)$ as the surrogate regret for a policy $\pi$.

In the non-stationary recommender system, all the methods perform nearly the same when the environment is stationary.
FTRL-PG has a slight edge over ONPG when the environment is stationary as all the past data is directly indicative of the future MDP.
It is interesting to note that while FTRL-PG works the best for the stationary setting in the recommender system and the goal reacher task, it is not the best in the diabetes treatment task as it can suffer from high variance.
We discuss the impact of variance in later paragraphs.

With the increase in the speed of non-stationarity, performance of both the baselines deteriorate quickly.
Of the two, ONPG is better able to mitigate performance lag as it discards all the past data.
In contrast, both the proposed methods, Pro-OLS and Pro-WLS, can leverage all the past data to better capture the impact of non-stationarity and thus are consistently more robust to the changes in the environment.

In the non-stationary goal reacher environment, a similar trend as above is observed.
While considering all the past data equally is useful for FTRL-PG in the stationary setting, it creates drastic performance lag as the speed of the non-stationarity increases.
 Between Pro-OLS and Pro-WLS, in the stationary setting, once the agent nearly solves the task all subsequent trajectories come from nearly the same distribution and thus the variance resulting from importance sampling ratio is not severe.
 In such a case, where the variance is low, Pro-WLS has less advantage over Pro-OLS and additionally suffers from being biased.
However, as the non-stationarity increases, the optimal policy keeps changing and there is a higher discrepancy between distributions of past and current trajectories.
This makes the lower variance property of Pro-WLS particularly useful.
Having the ability to better capture the underlying trend, both Pro-OLS and Pro-WLS consistently perform better than the baselines when there is non-stationarity.

The non-stationary diabetes treatment environment is particularly challenging as it has a continuous action set.
This makes importance sampling based estimators subject to much higher variance.
Consequently, Pro-OLS is not able to reliably capture the impact of non-stationarity and performs similar to FTRL-PG.
In comparison, ONPG is data-inefficient and performs poorly on this domain across all the speeds.
The most advantageous algorithm in this environment is Pro-WLS. 
Since Pro-WLS is designed to better tackle variance stemming from importance sampling, it is able to efficiently use the past data to capture the underlying trend and performs well across all the speeds of non-stationarity.

\section{Conclusion}

We presented a policy gradient-based algorithm that combines counter-factual reasoning with curve-fitting to proactively search for a good policy for future MDPs.
Irrespective of the environment being stationary or non-stationary, the proposed method can leverage all the past data, and in non-stationary settings it can  pro-actively optimize for future performance as well. 
Therefore, our method provides a single solution for mitigating performance lag and  being data-efficient.

While the proposed algorithm has several desired properties, many open questions remain.
In our experiments, we noticed that the proposed algorithm is particularly sensitive to the value of the entropy regularizer $\lambda$. 
Keeping $\lambda$ too high prevents the  policy from adapting quickly.
Keeping $\lambda$ too low lets the policy overfit to the forecast and become close to deterministic, thereby increasing the variance for subsequent importance sampling estimates of policy performance.
While we resorted to hyper-parameter search, leveraging  methods that adapt $\lambda$ automatically might be fruitful \citep{haarnoja2018soft}.

Our framework highlights new research directions for studying bias-variance trade-offs in the non-stationary setting.
While tackling the problem from the point of view of a univariate time-series is advantageous as the model-bias of the environment can be reduced, this can result in higher variance in the forecasted performance.   
Developing lower variance off-policy performance estimators is also an active research direction which directly complements our algorithm.
In particular, often a partial model of the environment is available and using it through doubly-robust estimators \citep{jiang2015doubly,thomas2016data} is an interesting future direction. 

Further, there are other forecasting functions, like kernel regression, Gaussian Processes, ARIMA, etc., and some break-point detection algorithms that can potentially be used to incorporate more domain knowledge in the forecasting function $\Psi$, or make $\Psi$ robust to jumps and auto-correlations in the time series.

\section{Acknowledgement}

Part of the work was done when the first author was an intern at Adobe Research, San Jose.
The research was later supported by generous gifts from Adobe Research.
We are thankful to Ian Gemp, Scott M. Jordan, and Chris Nota for insightful discussions and for providing valuable feedback.

Research reported in this paper was also sponsored in part by the CCDC Army Research Laboratory under Cooperative Agreement W911NF-17-2-0196 (ARL IoBT CRA). The views and conclusions contained in this document are those of the authors and should not be interpreted as representing the official policies, either expressed or implied, of the Army Research Laboratory or the U.S. Government. The U.S. Government is authorized to reproduce and distribute reprints for Government purposes notwithstanding any copyright notation herein.

\bibliography{mybib}
\bibliographystyle{icml2020}

\clearpage
\appendix
\setcounter{lemma}{0}
\setcounter{thm}{0}
\setcounter{cor}{0}
\setcounter{prop}{0}

\onecolumn

\icmltitle{Optimizing for the Future in Non-Stationary MDPs (Supplementary Material)}



\section{Proofs for the Properties of the NIS and NWIS Estimators}
\label{apx:sec:properties}

Here we provide proofs for the  properties of the NIS and NWIS estimators.
While NIS and NWIS are developed for the non-stationary setting, these properties ensure that these estimators 
generalize to the stationary setting as well.
That is, when used in a stationary setting, the NIS estimator is both unbiased and consistent like the PDIS estimator, and the NWIS estimator is biased and consistent like the WIS estimator.
%

Our proof technique draws inspiration from the results presented by \citet{mahmood2014weighted}.
%
%
The key modification that we make to leverage their proof technique is that instead of using the features of the state as the input and the observed return from that corresponding state as the target to the regression function, we use the features of the \textit{time index of an episode} as the input and the observed return for that corresponding episode as the target.
In their setup, because 
states are drawn stochastically 
from a distribution, their analysis is not directly applicable to our setting where 
inputs 
are time indices that form a deterministic sequence.
For analysis of our estimators, we leverage techniques discussed by \citet{greene2003econometric} for analyzing properties of the ordinary least squares estimator.

Before proceeding, we impose the following constraints on the set of policies, and the basis functions $\phi_i : \mathbb{N} \rightarrow \mathbb{R}$ used for encoding the time index in both $\Psi$ and $\Psi^\ddagger$, with $\phi(\cdot) = [\phi_1(\cdot), ..., \phi_{d-1}(\cdot), 1]$. 

%
    \textbf{(a)} $\phi(\cdot)$ always contains $1$ to incorporate a bias coefficient in least-squares regression (for example, $\phi(\cdot) = [\phi_1(\cdot), ..., \phi_{d-1}(\cdot), 1]$, where $\forall i \in [1, d-1]$, $\phi_i(\cdot)$ is a basis function).\footnote{If additional domain knowledge is available to select an appropriate basis function that can be used to represent the performance trend of all the policies for the given non-stationary environment, then all the following finite-sample and large-sample properties can be extended for that environment as well, using that basis function.}
%

\textbf{(b)} There exists a finite constant $C_1$, such that $\forall i, \,\, | \phi_i(\cdot)| < C_1$.

\textbf{(c)} $\Phi$ has full column rank such that $(\Phi^\top \Phi)^{-1}$ exists.

\textbf{(d)} We only consider a set of policies $\Pi$ that have non-zero probability of taking any action in any state. 
That is, $\exists C_2 > 0$, such that $\forall \pi \in \Pi, \forall s \in \mathcal S, \forall a \in \mathcal{A}, \pi(a|s) > C_2.$
%

Satisfying condition (a) is straightforward as it is typically already satisfied by all basis functions. 
This constraint ensures that the regression based forecasting function can capture a fixed constant that is required to model the absence of any trend.
This constraint is useful for our purpose as in the stationary setting there exists no trend in the expected performance across episodes for any given policy.

Conditions (b) and (c) are also readily satisfied by popular basis functions. 
For example, features from the Fourier basis are bounded by $[-1, 1]$, and features from polynomial/identity bases are also bounded when inputs are adequately normalized.
Further, when the basis function does not repeat any feature, and there are more samples than the number of features ($k \geq d$), condition (c) is satisfied.
This ensures that the least-squares problem is well-defined and has a unique-solution.

Condition (d)  ensures that the denominator in any importance ratio is always bounded below, such that the importance ratios are bounded above.
This implies that the importance sampling estimator for any policy has finite variance.
Use of entropy regularization with common policy parameterizations (softmax/Gaussian) can prevent violation of this condition.

In the following, we first establish the finite-sample properties and then we establish the large-sample properties for the NIS and NWIS estimators.
Before proceeding further, recall from \eqref{eqn:predict} and \eqref{eqn:predictweighted} that the NIS and NWIS estimators are given by:
%
\begin{align}
    \hat J_{k+\delta}(\pi) &=  \phi(k+\delta) w  = \phi(k+\delta)(\Phi^\top  \Phi)^{-1}\Phi^\top Y
    \\
    \hat J^\ddagger _{k+\delta}(\pi) &=  \phi(k+\delta) w^\ddagger  = \phi(k+\delta)(\Phi^\top \Lambda \Phi)^{-1}\Phi^\top \Lambda Y.
    \\
    \,
\end{align}

\subsection{Finite Sample Properties}
\label{apx:sub:finitesample}
In this subsection, finite sample properties of NIS and NWIS are presented.
Specifically, it is established that NIS is an unbiased estimator, whereas NWIS is a biased estimator of $J(\pi)$, where $J(\pi)$ is the performance of a policy $\pi$ in a stationary MDP.

\begin{thm}[Unbiased NIS] Under \thref{ass:stationary,ass:independentTraj}, for all $\delta \geq 1$, $\hat J_{k+\delta}(\pi)$ is an unbiased estimator of $J(\pi)$. That is, $\mathbb{E}[\hat J_{k+\delta}(\pi)] = J(\pi)$.
\thlabel{apx:thm:unbiasedNIS}
\end{thm} 
\begin{proof}
Recall from \eqref{eqn:predict} that
    \begin{align}
        \hat J_{k+\delta}(\pi) &= \phi(k+\delta)w = \phi(k+\delta) (\Phi^\top \Phi)^{-1} \Phi^\top Y.
    \end{align}
Therefore, the expected value of    $\hat J_{k+\delta}(\pi)$ is 
    \begin{align}
        \mathbb{E}[\hat J_{k+\delta}(\pi)] &= \mathbb{E}\left[ \phi(k+\delta) (\Phi^\top \Phi)^{-1} \Phi^\top Y \right]
        \\
        %
        %
        &=  \phi(k+\delta) \left(\Phi^\top \Phi \right)^{-1} \left (\Phi^\top  \mathbb{E}\left[ Y \right] \right ). \label{apx:eqn:ey} 
    \end{align}

As $Y = [\hat J_0(\pi), ..., \hat J_k(\pi)]^\top$ and the MDP is stationary, the expected value of each element of $Y$ is $J(\pi)$.
Further, since $\phi(\cdot)$ always contains the bias co-efficient, and the performance of any policy in invariant to the episode number in a stationary MDP (\thref{ass:stationary}), the optimal parameter for the regression model is $w^* = [0,0,...,0, J(\pi)]^\top$,  
such that for any $k$,
\begin{align}
    \phi(k)w^* = [\phi_1(k), ..., \phi_{d-1}(k), 1][0,...,0, J(\pi)]^\top = J(\pi) \label{apx:eqn:wstar}. 
\end{align}
Therefore, $\mathbb{E}[Y] = \Phi w^*$.
%
        %
Using this observation in \eqref{apx:eqn:ey},
\begin{align}
        \mathbb{E}[\hat J_{k+\delta}(\pi)] %
        &=  \phi(k+\delta) \left(\Phi^\top \Phi \right)^{-1} \left (\Phi^\top  \Phi w^* \right )
        \\
        &=  \phi(k+\delta) \left(\Phi^\top \Phi \right)^{-1} \left (\Phi ^\top \Phi \right)w^*
        \\
        &=  \phi(k+\delta)w^*
        \\
        &= J(\pi).
    \end{align}
\end{proof}

\begin{proof} (Alternate) Here we present an alternate proof for \thref{apx:thm:unbiasedNIS} which does not require invoking $w^*$.
    \begin{align}
        \mathbb{E}\left[\hat J_{k+\delta}(\pi)\right]
        &= \mathbb{E} \left [ \phi(k+\delta) (\Phi^\top \Phi)^{-1}\Phi^\top Y \right ]
        \\
        &\overset{(a)}{=} \mathbb{E}\left[\sum_{i=0}^k  \left[ \phi(k + \delta) (\Phi^\top \Phi)^{-1}\Phi^\top \right]_i Y_i \right]  \\
        &\overset{(b)}{=} \sum_{i=0}^k  \left[ \phi(k + \delta) (\Phi^\top \Phi)^{-1}\Phi^\top \right]_i   \mathbb{E}\left[ Y_i\right], 
    \end{align}
    where (a) is the dot product written as summation, and (b) holds because the multiplicative constants are fixed values, as given in \eqref{eqn:extra-grad}.
    Since the environment is stationary, $\forall i \,\, \mathbb{E}\left[ Y_i \right] = J(\pi)$, therefore,
    \begin{align}
        \footnotesize
         \mathbb{E}\left[\hat J_{k+1}(\pi)\right] &= J(\pi)  \sum_{i=0}^k  \left[ \phi(k + \delta) (\Phi^\top \Phi)^{-1}\Phi^\top \right]_i. \label{eqn:sumi}
    \end{align}
In the following we focus on the terms inside the summation in \eqref{eqn:sumi}.
Without loss of generality, assume that for a given matrix of features $\Phi$, the feature corresponding to value $1$ is in the last column of $\Phi$.
Let $\mathbf{A} := (\Phi^\top \Phi)^{-1}\Phi^\top \in \mathbb{R}^{d\times k}$ , and let $B := \Phi[1:k, 1:d-1 ] \in \mathbb{R}^{k \times (d-1)}$ be the submatrix of $\Phi$ such that $\mathbf{B}$ has all features of $\Phi$ except the ones column, $\mathbbm{1} \in \mathbb{R}^{k\times 1}$.
Let $\mathbf{I}$ be the identity matrix in $\mathbb{R}^{d\times d}$, then it can seen that $(\Phi^\top \Phi)^{-1}(\Phi^\top \Phi)$ can be expressed as,
\begin{align}
\left[
    \begin{array}{c}
        \mathbf{A} 
    \end{array}
\right]
\left[
    \begin{array}{c|c}
        \mathbf{B}   & \mathbbm{1} 
    \end{array}
\right ] = \mathbf{I}, \label{eqn:matrix}
\end{align}

This implies $[\mathbf{A}\mathbf{B} \,\,\,  \mathbf{A}\mathbbm{1}] = \mathbf{I}$. 
Therefore, as the $j^{th}$ row in last column of $\mathbf{I}$ corresponds to the dot product of the $j^{\text{th}}$ row of $\mathbf A$, $\mathbf{A_j}$, with $\mathbbm{1}$, 
\begin{align}
    \mathbf A_j \mathbbm{1} &= \begin{cases}
                                0 & j \neq d, \\
                                1 & j = d.
                            \end{cases} \label{eqn:sums}
\end{align}
 Equation \eqref{eqn:sums} ensures that the summation of all rows of $\mathbf A$, except the last, sum to $0$, and the last one sums to $1$.
 Now, let $\phi(k+\delta) := [\phi_1(k+\delta), \phi_2(k+\delta), ..., \phi_{d-1}(k+ \delta), 1] \in \mathbb{R}^{1 \times d}$.
 Therefore,
 \begin{align}
     \sum_{i=1}^k  \left[ \phi(k+\delta) (\Phi^\top \Phi)^{-1}\Phi^\top \right]_i &= \sum_{i=1}^k \left [ \phi(k+\delta) \mathbf{A} \right]_i \\ 
     &= \sum_{i=1}^k \sum_{j=1}^d \left[\phi(k+\delta)\right]_j \mathbf A_{j,i} \\
     &= \sum_{j=1}^d \left[\phi(k+\delta)\right]_j \sum_{i=1}^k \mathbf A_{j,i} \\
     &= \left(  \sum_{j=1}^{d-1} \left[\phi(k+\delta)\right]_j \sum_{i=1}^k \mathbf A_{j,i} \right)  + \left( \left[\phi(k+\delta)\right]_d \sum_{i=1}^k \mathbf A_{d,i}  \right )\\
     &= \left(  \sum_{j=1}^{d-1} \left[\phi(k+\delta)\right]_j (\mathbf A_{j}\mathbbm{1}) \right)  + \left( \left[\phi(k+\delta)\right]_d (\mathbf A_{d}\mathbbm{1})  \right )\\
     &= \left( \sum_{j=1}^{d-1} \left[\phi(k+\delta)\right]_j \cdot 0 \right) + \left( \left[\phi(k+\delta)\right]_d \cdot 1 \right)\\
     &= \left[\phi(k+\delta)\right]_d \\
     &= 1. \label{eqn:sum1}
 \end{align}
Therefore, combining \eqref{eqn:sum1} with  \eqref{eqn:sumi},
    $
         \mathbb{E}\left[\hat J_{k+\delta}(\pi)\right] = J(\pi). 
    $
\end{proof}

\begin{thm}[Biased NWIS] Under \thref{ass:stationary,ass:independentTraj}, for all $\delta \geq 1$, $\hat J^\ddagger _{k+\delta}(\pi)$ may be a \textit{biased} estimator of $J(\pi)$. That is, it is possible that $\mathbb{E}[\hat J^\ddagger _{k+\delta}(\pi)] \neq J(\pi)$.
\thlabel{apx:thm:unbiasedNWIS}
\end{thm} 

\begin{proof}
 We prove this result using a simple counter-example.
 Consider the following basis function, $\phi(\cdot) = [1]$: 
\begin{align}
 J^\ddagger_{k+\delta}(\pi) &= \phi(k+\delta)w^\ddagger\\
 &= \phi(k+\delta) \,\, \underset{c \in \mathbb{R}^{1 \times 1}}{\text{argmin}} \frac{1}{n}\sum_{i=1}^n \rho_i(0,T)(G_i - c \phi(i))^2 \\
  &= \underset{c \in \mathbb{R}^{1 \times 1}}{\text{argmin}} \frac{1}{n}\sum_{i=1}^n \rho_i(0,T)(G_i - c)^2
  \\
  &= \frac{\sum_{i=1}^n \rho_i(0,T)G_i}{\sum_{i=1}^n \rho_i(0,T)},
\end{align}
which is the WIS estimator.
Therefore, as WIS is a biased estimator \cite{precup2000eligibility}, NWIS is also a biased estimator of $J(\pi)$.
\end{proof}

\subsection{Large Sample Properties}
\label{apx:sub:largesample}

In this subsection, large sample properties of NIS and NWIS are presented.
Specifically, it is established that both NIS and NWIS are consistent estimators of, $J(\pi)$, the performance of a policy $\pi$ for a stationary MDP.


\begin{thm}[Consistent NIS] Under \thref{ass:stationary,ass:independentTraj}, for all $\delta \geq 1$, $\hat J_{k+\delta}(\pi)$ is a consistent estimator of $J(\pi)$. That is, as $N \rightarrow \infty,\, \hat J_{N+\delta}(\pi) \overset{\text{a.s.}}{\longrightarrow} J(\pi).$
\thlabel{thm:consistentNIS}
\end{thm} 

\begin{proof} The proof follows from the standard consistency result for ordinary least-squares regression \cite{greene2003econometric}.
Formally, using \eqref{eqn:predict}, 
    \begin{align}
        \underset{N \rightarrow \infty}{\lim} \quad \hat J_{N+\delta}(\pi) &= \underset{N \rightarrow \infty}{\lim} \quad \phi(N+\delta)w 
        \\
        &= \underset{N \rightarrow \infty}{\lim} \quad \phi(N+\delta) (\Phi^\top  \Phi)^{-1} \Phi^\top  Y.
    \end{align}
Since $Y = [\hat J_0(\pi), ..., \hat J_N(\pi)]^\top$ and the MDP is stationary,  each element of $Y$ is an unbiased estimate of $J(\pi)$.
In other words, $\forall i \in [0, N], \hat J_i(\pi) = J(\pi) + \epsilon_i$, where $\epsilon_i$ is a mean zero error.
Let $\epsilon \in \mathbb{R}^{N\times 1}$ be the vector containing all the error terms $\epsilon_i$.
Now, using \eqref{apx:eqn:wstar},
    \begin{align}
           %
        %
        \underset{N \rightarrow \infty}{\lim} \quad \hat J_{N+\delta}(\pi) &= \underset{N \rightarrow \infty}{\lim} \quad \phi(N+\delta) \left(\Phi^\top\Phi\right)^{-1} \left (\Phi^\top (\Phi w^* + \epsilon) \right)  \label{apx:eqn:wstarpluserror}
        \\
        &= \underset{N \rightarrow \infty}{\lim} \quad \phi(N+\delta) \left(\Phi^\top \Phi\right)^{-1} \left ( \left(\Phi^\top \Phi \right)w^* + \left(\Phi^\top \epsilon \right) \right)
        \\
        &= \underset{N \rightarrow \infty}{\lim} \quad \phi(N+\delta)w^* + \phi(N+\delta) \left(\Phi^\top \Phi\right)^{-1} \left (\Phi^\top \epsilon \right)
        \\
        &= \underset{N \rightarrow \infty}{\lim} \quad  J(\pi) + \phi(N+\delta) \left(\Phi^\top \Phi\right)^{-1} \left (\Phi^\top \epsilon \right)
        \\
        &=  \underset{N \rightarrow \infty}{\lim} \quad J(\pi) + \phi(N+\delta) \left(\frac{1}{N} \Phi^\top\Phi\right)^{-1} \left (\frac{1}{N} \Phi^\top \epsilon \right)       .
\end{align}
If both $Q^{-1} \coloneqq \left(\underset{N \rightarrow \infty}{\lim}\frac{1}{N} \Phi^\top\Phi\right)^{-1}$ and $\left(\underset{N \rightarrow \infty}{\lim}\frac{1}{N} \Phi^\top \epsilon\right)$ exist, then using Slutsky's Theorem,
\begin{align}
        \underset{N \rightarrow \infty}{\lim} \quad \hat J_{N+\delta}(\pi)  &=  J(\pi) + \phi(N+\delta) Q^{-1} \left (\underset{N \rightarrow \infty}{\lim} \quad \frac{1}{N} \Phi^\top \epsilon \right) \label{apx:eqn:decomperror}.
        \end{align}
To validate  conditions for Slutsky's Theorem, notice that it holds from Grenander's conditions that $Q^{-1}$ exists.
Informally, Grenander's conditions require that no feature degenerates to a sequence of zeros, no feature of a single observation dominates the sum of squares of its series, and the $\Phi^\top\Phi$ matrix always has full rank.
These conditions are easily satisfied for most popular basis functions used to create input features.
For formal definitions of these conditions, we refer the reader to Chpt. 5, \citet{greene2003econometric}. 

In the following, we restrict our focus to the term inside the brackets in the second term of \eqref{apx:eqn:decomperror} and show that it exists, so that \eqref{apx:eqn:decomperror} is valid.
Notice that the mean of that term is,
\begin{align}
    \mathbb{E}\left[ 
        \frac{1}{N}\Phi^\top \epsilon  \right] &=      
        \frac{1}{N} \Phi^\top \mathbb{E}\left[\epsilon  \right] = 0.
\end{align}

Since the mean is $0$, the variance is,
\begin{align}
    \mathbb{V}\left[ 
        \frac{1}{N} \Phi^\top \epsilon  \right] &=      
        \frac{1}{N^2}\mathbb{V}\left[ \Phi^\top \epsilon  \right] 
        = \frac{1}{N^2} \mathbb{E} \left[ \left( \Phi^\top \epsilon \right) \left(\Phi^\top \epsilon \right)^\top \right]
        = \frac{1}{N^2} \left( \Phi^\top \mathbb{E}\left [ \epsilon \epsilon^\top | \Phi\right] \Phi \right).
\end{align}
As each policy has a non-zero probability of taking any action in any state, the variance of PDIS (or the standard IS) estimator is bounded and thus each element of $\mathbb{E}[\epsilon \epsilon^\top|\Phi]$ is bounded.
Further, as $\phi_i(\cdot)$ is bounded, each element of $\Phi$ is also bounded.
Therefore,  
\begin{align}
\underset{N \rightarrow \infty}{\lim} \quad \mathbb{V}\left[ 
        \frac{1}{N} \Phi^\top \epsilon  \right]  \rightarrow 0. \label{apx:eqn:zerovar}
\end{align}
Since the mean is $0$ and the variance asymptotes to $0$, by  Kolmogorov's strong law of large numbers it follows that 
 as $N\rightarrow \infty,\,\, \frac{1}{N} \Phi^\top \epsilon  \overset{\text{a.s.}}{\longrightarrow} 0$. 
Combining this with \eqref{apx:eqn:decomperror}, 
\begin{align}
    \underset{N \rightarrow \infty}{\lim} \,\, \hat J_{N+\delta}(\pi) &\overset{a.s.}{\rightarrow}  J(\pi) + \phi(N+\delta) Q^{-1} 0 = J(\pi).
\end{align}

\end{proof}

\begin{thm}[Consistent NWIS] Under \thref{ass:stationary,ass:independentTraj}, for all $\delta \geq 1$, $\hat J_{k+\delta}^\ddagger(\pi)$ is a consistent estimator of $J(\pi)$. That is, as $N \rightarrow \infty, \, \hat J_{N+\delta}^\ddagger(\pi)  \overset{\text{a.s.}}{\longrightarrow} J(\pi).$
\thlabel{thm:consistentNWIS}
\end{thm} 

\begin{proof}
Recall from \eqref{eqn:weightedLS} that
    \begin{align}
        \hat J_{N+\delta}^\ddagger(\pi) &= \phi(N+\delta)w^\ddagger = \phi(N+\delta) (\Phi^\top \Lambda \Phi)^{-1} \Phi^\top \Lambda Y \label{apx:eqn:NWISrecast}.
    \end{align}
Consistency of $\hat J_{N+\delta}^\ddagger(\pi)$ can be proven similarly to the proof of \thref{thm:consistentNIS}.
Note that here $Y =  [G_0, ..., G_k]^\top $ contains the returns for each episode, and $\Lambda Y$ denotes the unbiased estimates for $J(\pi)$.
Therefore, similar to \eqref{apx:eqn:wstarpluserror},
\begin{align}
     \underset{N \rightarrow \infty}{\lim} \quad \hat J_{N+\delta}^\ddagger(\pi) &=  \underset{N \rightarrow \infty}{\lim} \quad \phi(N+\delta) (\Phi^\top \Lambda \Phi)^{-1} (\Phi^\top (\Phi w^* + \epsilon))
     \\
     &=\underset{N \rightarrow \infty}{\lim}\quad \phi(N+\delta) (\Phi^\top \Lambda \Phi)^{-1} ((\Phi^\top \Phi) w^* + \Phi^\top \epsilon)
     \\
     &=\underset{N \rightarrow \infty}{\lim}\quad \phi(N+\delta) \left(\frac{1}{N} \Phi^\top \Lambda \Phi \right)^{-1} \left (\left(\frac{1}{N}\Phi^\top \Phi \right) w^* + \frac{1}{N} \Phi^\top \epsilon\right). \label{apx:eqn:WLSlim}
\end{align}
In the following, we will make use of Slutsky's Theorem.
To do so, we first restrict our focus to the terms in the first bracket in \eqref{apx:eqn:WLSlim}, and show existence of its limit.
Let $\tilde \rho_k \coloneqq \rho^\ddagger_k - \mathbb{E}[\rho^\ddagger_k]$ be a mean $0$ random variable, then
\begin{align}
    \underset{N \rightarrow \infty}{\lim}\quad\frac{1}{N} \Phi^\top \Lambda \Phi  &= \underset{N \rightarrow \infty}{\lim}\quad\frac{1}{N} \sum_{k=1}^{N}\rho^\ddagger_k \phi(k)^\top \phi(k). 
    \\
    &= \underset{N \rightarrow \infty}{\lim}\quad \left(\frac{1}{N} \sum_{k=1}^N \tilde \rho_k \phi(k)^\top \phi(k) +   \frac{1}{N} \sum_{k=1}^N \mathbb{E}\left[\rho_k^\ddagger\right] \phi(k)^\top \phi(k) \right). 
    \\
    &\overset{(a)}{\rightarrow}   \underset{N \rightarrow \infty}{\lim}\quad \frac{1}{N}  \sum_{k=1}^{N} \mathbb{E}\left[\rho^\ddagger_k \right]  \phi(k)^\top \phi(k)  
    \\
    &\overset{(b)}{=}   \underset{N \rightarrow \infty}{\lim}\quad \frac{1}{N}  \sum_{k=1}^{N} \phi(k)^\top \phi(k)  
    \\
    %
    %
    &=   \underset{N \rightarrow \infty}{\lim}\quad\frac{1}{N} \Phi^\top \Phi , \label{apx:eqn:rhoeliminated}
\end{align}
where (a) follows from the Kolmogorov's strong law of large numbers.
To see this, let $Z_k = \tilde \rho_k\phi(k)^\top \phi(k)$.
Notice that $\mathbb{E}[Z_k] = \mathbb{E}[\tilde \rho_k]\phi(k)^\top \phi(k) = 0$, and as
both $\tilde \rho$ and $\phi(\cdot)$ are bounded, the variance of $Z_k$ is also bounded.
Therefore, $(1/N) \sum \tilde \rho_k \phi(k)^\top \phi(k) \rightarrow 0 $ almost surely as $N \rightarrow \infty$.
%
%
Consequently, (b) is obtained using the fact that the expected value of importance ratios is $1$ \citep[Lemma 3]{thomas2015safe}.
Notice that \eqref{apx:eqn:rhoeliminated} reduced to $Q$ (which was defined in the proof of \thref{thm:consistentNIS}) and we know that its limit exists because $\phi(\cdot)$ is bounded. 
Further, we also know that $Q^{-1}$ and $\left(\underset{N \rightarrow \infty}{\lim}\frac{1}{N} \Phi^\top \epsilon\right)$ exist (see the proof of \thref{thm:consistentNIS}).
Therefore, using Slutsky's Theorem and substituting  \eqref{apx:eqn:rhoeliminated} in \eqref{apx:eqn:WLSlim},

%
\begin{align}
    \underset{N \rightarrow \infty}{\lim} \quad \hat J_{N+\delta}^\ddagger(\pi)
     &= \phi(N+\delta) \left(\underset{N \rightarrow \infty}{\lim}\quad\frac{1}{N} \Phi^\top \Phi \right)^{-1} \left ( \left(\underset{N \rightarrow \infty}{\lim}\quad\frac{1}{N}\Phi^\top \Phi \right) w^* + \underset{N \rightarrow \infty}{\lim}\quad \frac{1}{N} \Phi^\top \epsilon\right)
     \\
     &= \phi(N+\delta) w^* + \phi(N+\delta)\left(\underset{N \rightarrow \infty}{\lim}\quad\frac{1}{N} \Phi^\top \Phi \right)^{-1}\left( \underset{N \rightarrow \infty}{\lim}\quad \frac{1}{N} \Phi^\top \epsilon\right)
     \\
     &= J(\pi) + \phi(N+\delta)\left(\underset{N \rightarrow \infty}{\lim}\quad\frac{1}{N} \Phi^\top \Phi \right)^{-1}\left( \underset{N \rightarrow \infty}{\lim}\quad \frac{1}{N} \Phi^\top \epsilon\right)
     \\
    & \overset{\text{a.s.}}{\longrightarrow} J(\pi) \label{apx:eqn:WLSfinal},    
\end{align}
where \eqref{apx:eqn:WLSfinal} follows from the simplification used for \eqref{apx:eqn:decomperror} in the proof of \thref{thm:consistentNIS}. 
\end{proof}

\section{Gradient of PDIS  Estimator}

\label{apx:sec:gradients}

Recall that the NIS and NWIS estimators build upon  estimators of past performances by using them along with OLS and WLS regression, respectively.
Consequently, gradients of the NIS and NWIS estimators with respect to the policy parameters can be decomposed into terms that consist of gradients of the estimates of past performances with respect to the policy parameters.
Here we provide complete derivations for obtaining a straightforward equation for computing the gradients of the PDIS estimator with respect to the policy parameters.
These might also be of independent interest when dealing with off-policy policy optimization for stationary MDPs.

\begin{prop}[PDIS Gradient]
Let $\rho_i(0,l) \coloneqq \prod_{j=0}^{l} \frac{\pi^\theta(A_i^j|S_i^j)}{\beta_i(A_i^j|S_i^j)}$, then
\begin{align}
    \nabla \hat J_i(\theta) &= \sum_{t=0}^{T}  \frac{  \partial \log  \pi^\theta(A^t_i|S^t_i)}{\partial \theta}  \left( \sum_{l=t}^{T} \rho_i(0, l)  \gamma^l  R^l_i  \right) .
\end{align}
\end{prop}
\begin{proof} 
Recall from \eqref{eqn:ope} that,
\begin{align}
    \hat J_i(\theta) &= \sum_{t=0}^{T}\left(\prod_{l=0}^{t} \frac{\pi^\theta(A_i^l|S_i^l)}{\beta(A_i^l|S_i^l)}\right) \gamma^t R_i^t .
    \end{align}
    Computing the gradient of $\hat J_i(\theta)$,
    \begin{align}
    \nabla \hat J_i(\theta) &= \sum_{t=0}^{T} \frac{\partial}{\partial \theta} \left(\prod_{l=0}^{t} \frac{\pi^\theta(A_i^l|S_i^l)}{\beta(A_i^l|S_i^l)}\right) \gamma^t R_i^t\\
     &=\sum_{t=0}^{T} \left(\prod_{l=0}^{t} \frac{\pi^\theta(A_i^l|S_i^l)}{\beta(A_i^l|S_i^l)}\right) \frac{\partial \log \left(\prod_{l=0}^{t} \pi^\theta(A_i^l|S_i^l) \right)}{\partial \theta} \gamma^t R_i^t  \\
     &=\sum_{t=0}^{T} \left(\prod_{l=0}^{t} \frac{\pi^\theta(A_i^l|S_i^l)}{\beta(A_i^l|S_i^l)}\right) \left( \sum_{l=0}^{t} \frac{  \partial \log  \pi^\theta(A_i^l|S_i^l)}{\partial \theta} \right)\gamma^t R_i^t  \\
     &= \sum_{t=0}^{T} \rho_i(0, t) \left( \sum_{l=0}^{t} \frac{  \partial \log  \pi^\theta(A_i^l|S_i^l)}{\partial \theta} \right)\gamma^t R_i^t . \label{eqn:opg1} \\
     &= \sum_{t=0}^{T}  \frac{  \partial \log  \pi^\theta(A_i^t|S_i^t)}{\partial \theta}  \left( \sum_{l=t}^{T} \rho_i(0, l)  \gamma^l  R_i^l  \right)  \label{eqn:opg2},
\end{align}

\begin{table}[t]
\begin{center}
\begin{tabular}{l|lllll}
 ${}_t$ \textbackslash ${}^l$ &  0 &  1  &  2 & ... &  T \\
 \hline
 0 & $\gamma^0 \rho_i(0, 0) \Psi_i^0 R_i^0$  & \cellcolor{lightgray} & \cellcolor{lightgray} & \cellcolor{lightgray} &\cellcolor{lightgray}  \\
 1 & $\gamma^1 \rho_i(0, 1) \Psi_i^0 R_i^1$ & $\gamma^1 \rho_i(0, 1) \Psi_i^1 R_i^1$ & \cellcolor{lightgray} & \cellcolor{lightgray} & \cellcolor{lightgray}\\
 2 & $\gamma^2 \rho_i(0, 2) \Psi_i^0 R_i^2$ & $\gamma^2 \rho_i(0, 2) \Psi_i^1 R_i^2$ & $\gamma^2 \rho_i(0, 2) \Psi_i^2 R_i^2$ & \cellcolor{lightgray} & \cellcolor{lightgray}\\
 \vdots & \vdots & \vdots & \vdots & \vdots & \cellcolor{lightgray}\vdots \\
 T & $\gamma^T \rho_i(0, T) \Psi_i^0 R_i^T$ & $\gamma^T \rho_i(0, T) \Psi_i^1 R_i^T$ & $\gamma^T \rho_i(0, T) \Psi_i^2 R_i^T$ & ... & $\gamma^T \rho_i(0, T) \Psi_i^T R_i^T$ 
\end{tabular}
\end{center}
\caption{let $\Psi_i^t = \partial \log  \pi^\theta(A_i^t|S_i^t)/ \partial \theta$. This table represents all the terms in \eqref{eqn:opg2} required for computing $\nabla \hat J_i(\theta)$. Gray color denotes empty cells.}
\label{tab:offpg}
\end{table}
%
%
where, in the last step, instead of the summation over the partial derivatives of $\log \pi^\theta$ for each weight $\rho(\cdot, \cdot)$, we consider the alternate form where the summation is over the importance weights $\rho (\cdot, \cdot)$ for each partial derivative of $\log \pi^\theta$. 
To see this step clearly, let $\Psi_i^t = \partial \log  \pi^\theta(A_i^t|S_i^t)/ \partial \theta$, then Table \ref{tab:offpg} shows all the terms in  \eqref{eqn:opg2}.
The last step above  corresponds to taking the column-wise sum instead of the row-wise sum in Table \ref{tab:offpg}.
\end{proof}

\section{Detailed Literature Review}
\label{apx:sec:litreview}
%

%
The problem of non-stationarity has a long history.
In the operations research community, many dynamic sequential decision-making problems are modeled using infinite horizon \textit{non-homogeneous} MDPs \citep{hopp1987new}.
While estimating an optimal policy is infeasible under an infinite horizon setting when the dynamics are changing and a stationary distribution cannot be reached, several researchers have studied the problem of  identifying sufficient forecast horizons for performing near-optimal planning \citep{garcia2000solving,cheevaprawatdomrong2007solution,ghate2013linear} or robust policy iteration \citep{sinha2016policy}.
%

%
%
In contrast, non-stationary multi-armed bandits (NMAB) capture the setting where the horizon length is one, but the reward distribution changes over time \citep{moulines2008,besbes2014stochastic}.
%
%
%
Many variants of NMAB, like \textit{cascading non-stationary bandits} \citep{wang2019aware,li2019cascading} and  \textit{rotten bandits} \citep{levine2017rotting,seznec2018rotting} have also been considered.
In optimistic online convex optimization, researchers have shown that better performance can be achieved by updating the parameters using predictions (which are based on the past gradients) of the gradient of the future loss  \citep{rakhlin2013online, yang2016optimistic, mohri2016accelerating,wang2019optimistic}. 
%
%
%

Non-stationarity also occurs in multiplayer games, like rock-paper-scissors, where each episode is a single one-step interaction \citep{singh2000nash,bowling2005convergence,conitzer2007awesome}.
Opponent modeling in games has been shown to be useful and regret bounds for multi-player games where players can be replaced with some probability $p$, i.e., the game changes slowly over time, have also been established \citep{zhang2010multi,mealing2013opponent,foster2016learning, foerster2018learning}.
%
%
%
However, learning sequential strategies in a non-stationary setting is still an open research problem.
%

%
%

%
For episodic non-stationary MDPs, researchers have also looked at providing regret bounds for algorithms that exploit oracle access to the current reward and transition functions \citep{even2005experts,yu2009online,abbasi2013online,lecarpentier2019non,ying2019}.
%
%
Alleviating oracle access by performing a count-based estimate of the reward and transition functions based on the recent history of interactions has also been proposed \citep{gajane2018sliding,cheung2019reinforcement}.
For tabular MDPs, past data from a non-stationary MDP can be used to construct a maximum-likelihood estimate model \citep{ornik2019learning} or a full Bayesian model \citep{jong2005bayesian} of the transition dynamics.
%
%
Our focus is on the setting which is not restricted to tabular representations.  

%
A Hidden-Mode MDP is an alternate setting that assumes that the environment changes are confined to a small number of hidden modes, where each mode represents a unique MDP.  
This provides a more tractable way to model a limited number of MDPs 
\citep{choi2000environment,basso2009reinforcement,mahmud2013learning},
or perform model-free updates using mode-change detection  \citep{padakandla2019reinforcement}.
%
%
In this work, we are interested in the continuously changing setting, where the number of possible MDPs is unbounded.
%

%
%
Tracking has also been shown to play an important role in non-stationary domains.
\citet{thomas2017predictive} and \citet{jagerman2019when} have proposed policy evaluation techniques in a non-stationary setting by tracking a policy's past performances.
However, they do not provide any procedure for searching for a good future policy.
To adapt quickly in non-stationary tasks, TIDBD \citep{kearney2018tidbd} and AdaGain \citep{jacobsen2019meta} perform TD-learning while also automatically (de-)emphasizing updates to (ir)relevant features by modulating the learning rate of the parameters associated with the respective features.
Similarly, \citet{abdallah2016addressing} propose repeating a Q-value update inversely proportional to the probability with which an action was chosen to obtain a transition tuple. 
%
%
%
In this work, we go beyond tracking and  proactively optimize for the future.

%
%

\section{Implementation details}
\label{apx:sec:implementation}

\subsection{Environments}
\label{apx:sub:env}

We provide empirical results on three non-stationary environments: diabetes treatment, recommender system, and a goal-reacher task.
Details for each of these environments are provided in this section.

\textbf{Non-stationary Diabetes Treatment:}
This MDP models the problem of Type-1 diabetes management.
A person suffering from Type-1 diabetes does not produce enough \textit{insulin}, a hormone that promotes absorption of glucose from the blood.
%
%
%
Consumption of a meal increases the blood-glucose level in the body, and if the blood-glucose level becomes too high, then the patient can suffer from \textit{hyperglycemia}.
Insulin injections can  reduce the blood-glucose level, but if the level becomes too low, then the patient suffers from \textit{hypoglycemia}.
While either of the extremes is undesirable, hypoglycemia is more dangerous and can triple the five-year mortality rate for a person with diabetes \citep{man2014uva}.

Autonomous medical support systems have been proposed to decide how much insulin should be injected to keep a person's blood glucose levels near ideal levels \citep{bastani2014model}. 
Currently, the parameters of such a medical support system are set by a doctor specifically for each patient.
However, due to non-stationarities induced over time as a consequence of changes in the body mass index, the insulin sensitivity of the pancreas, diet, etc., the parameters of the controller need to be readjusted regularly.
Currently, this requires revisiting the doctor. 
A viable reinforcement learning solution to this non-stationary problem could enable the automatic tuning of these parameters for patients who lack 
regular access to a physician.

To model this MDP, we use an open-source implementation \citep{simglucose} of the U.S. Food and Drug Administration (FDA) approved Type-1 Diabetes Mellitus simulator (T1DMS) \citep{man2014uva} for treatment of Type-1 diabetes,
where we induce non-stationarity by oscillating the body parameters between two known configurations. 
Each episode consists of a day ($1440$ timesteps, where each timestep corresponds to a minute) in an \textit{in-silico} patient's life and the transition dynamics of a patient's body for each second is governed by a continuous time ordinary differential equation (ODE) \citep{man2014uva}.
After each minute the insulin controller is used to inject the desired amount of insulin for controlling the blood glucose.

For controlling the insulin injection, we use a parameterized policy based on the amount of insulin that a person with diabetes is instructed to inject prior to eating a meal \citep{bastani2014model}:
\begin{align}
    \text{injection} = \frac{\text{current blood glucose} - \text{target blood glucose}}{CF} + \frac{\text{meal size}}{CR},
\end{align}
where `current blood glucose' is the estimate of the person's current blood-glucose level, `target blood glucose' is the desired blood glucose, `meal size' is the estimate of the size of the meal the patient is about to eat, and $CR$ and $CF$ are two real-valued parameters, that must be tuned based on the body parameters to make the treatment effective.
%
%

\textbf{Non-stationary Recommender System:}
During online recommendation of movies, tutorials, advertisements and other products, a recommender system needs to interact and personalize for each user.
However, interests of a user for different items, among the products that can be recommended, fluctuate over time.
For example, interests during online shopping can vary based on seasonality or other unknown factors.

This environment models the desired recommender system setting where reward (interest of the user) associated with each item changes over time.
Figure \ref{apx:fig:recoresults} (left) shows how the reward associated with each item changes over time, for each of the considered `speeds' of non-stationarity. 
The goal for the reinforcement learning agent is to maximize revenue by recommending the item which the user is most interested in at any time.

\textbf{Non-stationary Goal Reacher:}
For an autonomous robot dealing with tasks in the open-world, it is natural for the problem specification to change over time.
An ideal system should quickly adapt to the changes and still complete the task. 

To model the above setting, this environment considers a task of reaching a non-stationary goal position. 
That is, the location of the goal position keeps slowly moving around with time.
The goal of the reinforcement learning agent is to control the four (left, right, up, and down) actions to move the agent towards the goal as quickly as possible given the real valued Cartesian coordinates of the agent's current location.
The maximum time given to the agent to reach the goal is 15 steps.
%






\subsection{Hyper-parameters}
\label{apx:sub:hyper}

For both the variants of the proposed \textit{Prognosticator} algorithms, we use the Fourier basis to encode the time index while performing (ordinary/weighted) least squares estimation.
%
%
%
Since the Fourier basis requires inputs to be normalized with $|x| \leq 1$, we normalize each time index by dividing it by $K + \delta$, where $K$ is the current time and $\delta$ is the maximum time into the future that we will forecast for.
Further, as we are regressing only on time (which are all positive values), is does not matter whether the function for the policy performance over time is odd ($\Psi(x) = - \Psi(-x)$) or not.
Therefore, we drop all the terms in the basis corresponding to $\sin(\cdot)$, which are useful for modeling odd functions.
This reduces the number of parameters to be estimated by half.
Finally, instead of letting $n \in \mathbb{N}$, we restrict it to a finite set $\{1, ..., d-1\}$, where $d$ is a fixed constant that determines the size of the feature vector for each input.
In all our experiments, $d$ was a hyper-parameter chosen from $\{3, 5, 7\}$.

Other hyper-parameter ranges were common for all the algorithms.
The discounting factor $\gamma$ was kept fixed to $0.99$ and learning rate $\eta$ was chosen from the range $[5 \times 10^{-5}, 5\times 10^{-2}]$.
The entropy regularizer $\lambda$ was chosen from the range $[0, 1 \times 10^{-2}]$.
The batch size $\delta$ was chosen from the set $\{1, 3, 5\}$.
Inner optimization over past data for the proposed methods and FTRL-PG was run for $\{10, 20, 30\} \times \delta$ iterations.
Inner optimization for ONPG corresponds to one iteration over all the trajectories collected in the current batch.
Past algorithms have shown that clipping the importance weights can improve stability of reinforcement learning algorithms \citep{schulman2017proximal}.
Similarly, we clip the maximum value of the importance ratio to a value chosen from $\{5, 10, 15\}$.
%
%
%
As the non-stationary diabetes treatment problem has a continuous action space, the policy was parameterized with a Gaussian distribution having a variance chosen from $[0.5, 2.5]$.
For the non-stationary goal-reacher environment, the policy was parameterized using a two-layer neural network with number of hidden nodes chosen from $\{16, 32, 64\}$. 

In total, $2000$ settings for each algorithm, for each domain, were uniformly sampled (loguniformly for learning rates and $\lambda$) from the mentioned hyper-parameter ranges/sets.
Results from the best performing settings are reported in all plots.
Each hyper-parameter setting was run using $10$ seeds for the non-stationary diabetes treatment (as it was time intensive to run a continuous time ODE for each step) and $30$ seeds for the other two environments to get the standard error of the mean performances.
The authors had shared access to a computing cluster, consisting of 50 compute nodes with 28 cores each, which was used to run all the experiments.\footnote{
Code for our algorithm can be accessed using the following link: \href{https://github.com/yashchandak/OptFuture_NSMDP}{https://github.com/yashchandak/OptFuture\_NSMDP}.}

\section{Detailed Empirical Results}
\label{apx:sec:results}

\paragraph{Complexity analysis (space, time, and sample size)} %
The space requirement for our algorithms and FTRL-PG is linear in the number of episodes seen in the past, whereas it is constant for ONPG as it discards all the past data.
The computational cost of our algorithm is also similar to FTRL-PG as the only additional cost is that of differentiating through least-squares estimators which involves computing $(\Phi^\top \Phi)^{-1}$ or $(\Phi^\top\Lambda \Phi)^{-1}$. 
This additional overhead is negligible as these matrices are of the size $d \times d$, where $d$ is the size of the feature vector for time index and $d << N$, where $N$ is the number of past episodes.
Figures \ref{fig:results}, \ref{apx:fig:recoresults}, and \ref{apx:fig:detailedresults} present an empirical estimate for the sample efficiency.

\paragraph{Ablation study:} In Figure \ref{Fig:ablation}, we show the impact of basis function,  $\phi(\cdot)$, on the performance of both of our proposed algorithms: Pro-OLS and Pro-WLS.
Dimension $d$ for both the Fourier basis and the polynomial basis was chosen from $\{3,5,7\}$.
All other hyper-parameters were searched as  described in Section \ref{apx:sub:hyper}.
It can be seen that both the Fourier and polynomial basis functions provide sufficient flexibility for modeling the trend, whereas linear basis offers limited flexibility and results in poor performance.

\begin{figure}
    \centering
    \includegraphics[width=0.3\textwidth]{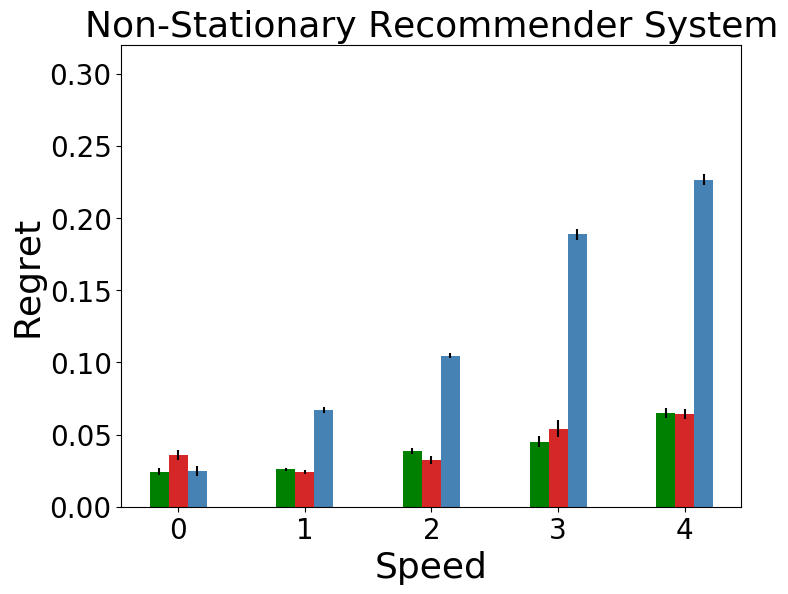}
    \hspace{18pt}
    \includegraphics[width=0.3\textwidth]{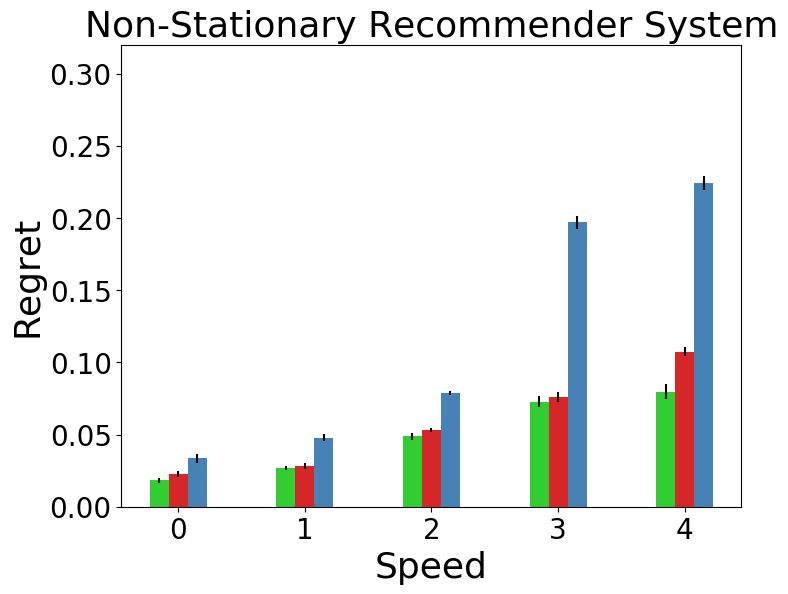}
    \\
    \includegraphics[width=0.3\textwidth]{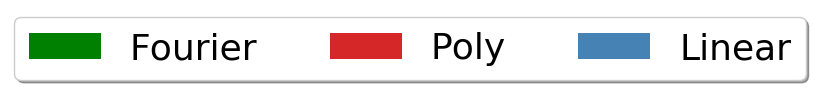}
    \hspace{20pt}
    \includegraphics[width=0.3\textwidth]{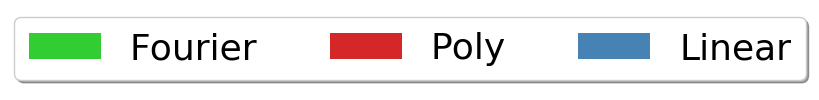}
    \caption{   Best performances of all the algorithms for the non-stationary recommender system environment, obtained by conducting a hyper-parameter sweep over $1000$ hyper-parameter combinations per algorithm.
    For each hyper-parameter setting, $30$ trials were executed.
    Error bars correspond to the standard error.
    (Left) Performance of Pro-OLS with Fourier, polynomial, and linear basis functions. (Right) Performance of Pro-WLS with Fourier, polynomial, and linear basis functions. 
    }
    %
    \label{Fig:ablation}
\end{figure}

\paragraph{Performance over time:} 
In Figure \ref{fig:results}, summary statistics of the results were presented.
In this section we present all the results in detail.
Figure \ref{apx:fig:recoresults} shows the performances of all the algorithms for individual episodes as the user interests changes over time in the recommender system environment.
In this environment, as the true reward for each of the items is directly available, we provide a visual plot for it as well in Figure \ref{apx:fig:recoresults} (left).
Notice that the shape of the performance curve for the proposed methods  closely resembles the trend of the maximum reward attainable across time.

Figure \ref{apx:fig:detailedresults} shows the performances of all the algorithms for the non-stationary goal-reacher and the diabetes treatment environments.
In these environments, the maximum achievable performance for each episode is not readily available. 

\begin{figure}
    \centering
    \includegraphics[width=0.3\textwidth]{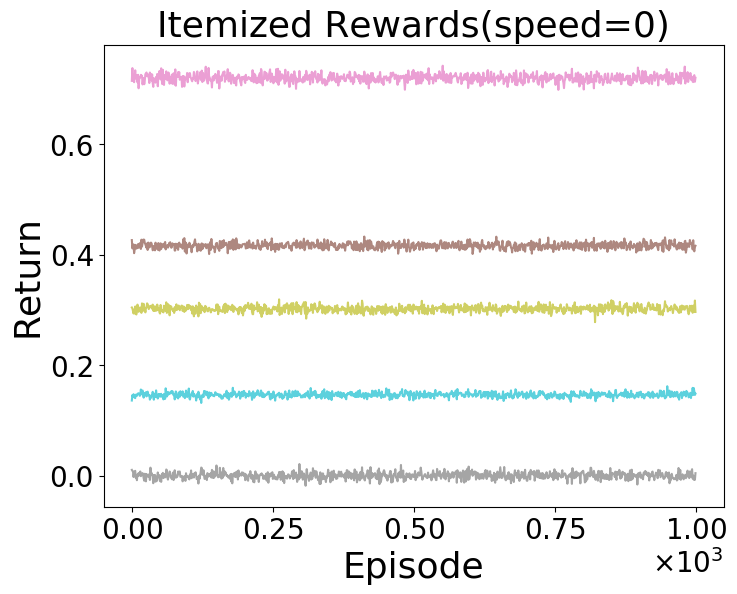} \quad\quad\quad\quad
    \includegraphics[width=0.3\textwidth]{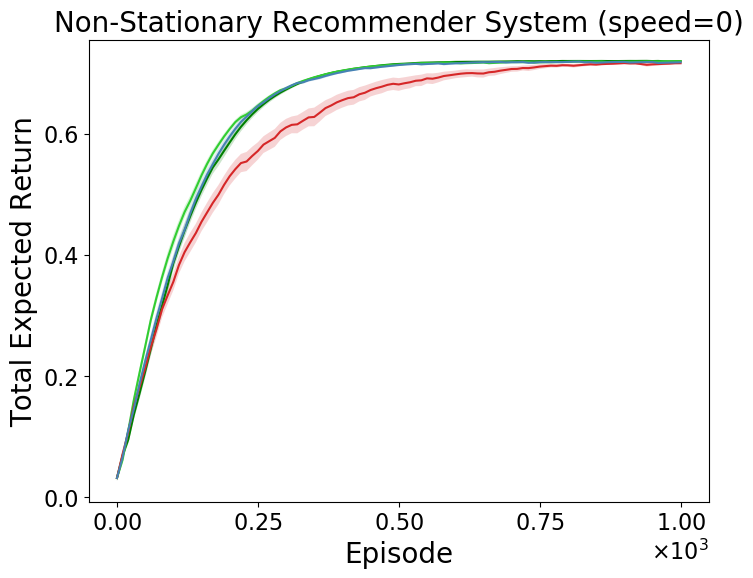}
    \\
    \includegraphics[width=0.3\textwidth]{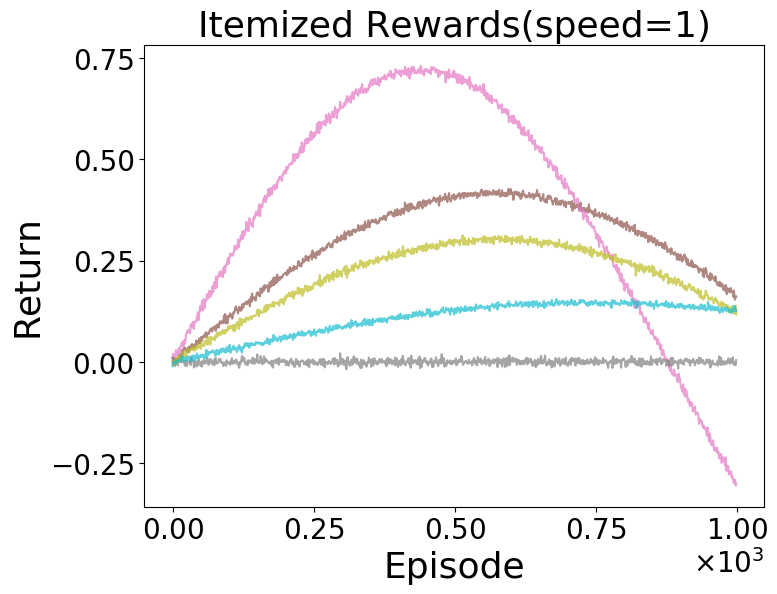} \quad\quad\quad\quad
    \includegraphics[width=0.3\textwidth]{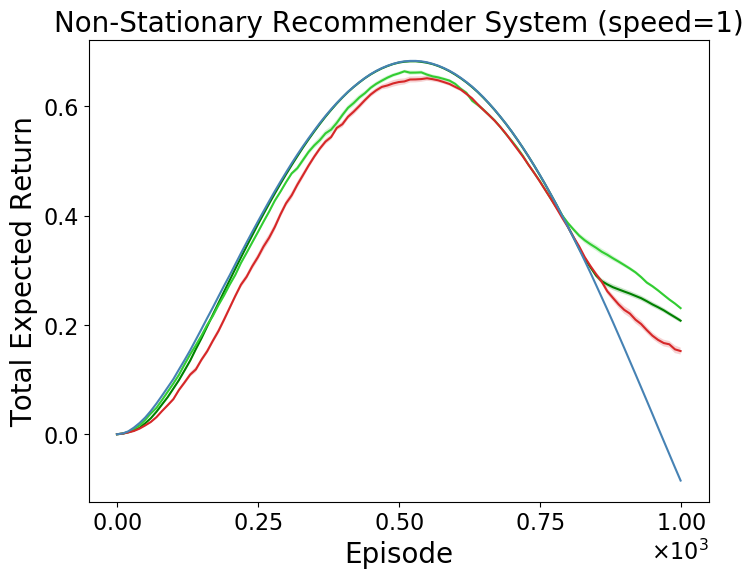}\\
    \includegraphics[width=0.3\textwidth]{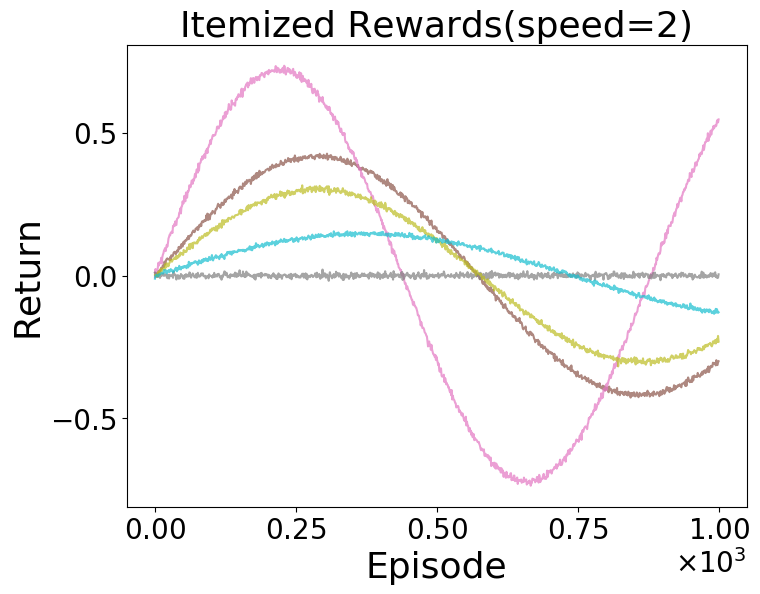} \quad\quad\quad\quad
    \includegraphics[width=0.3\textwidth]{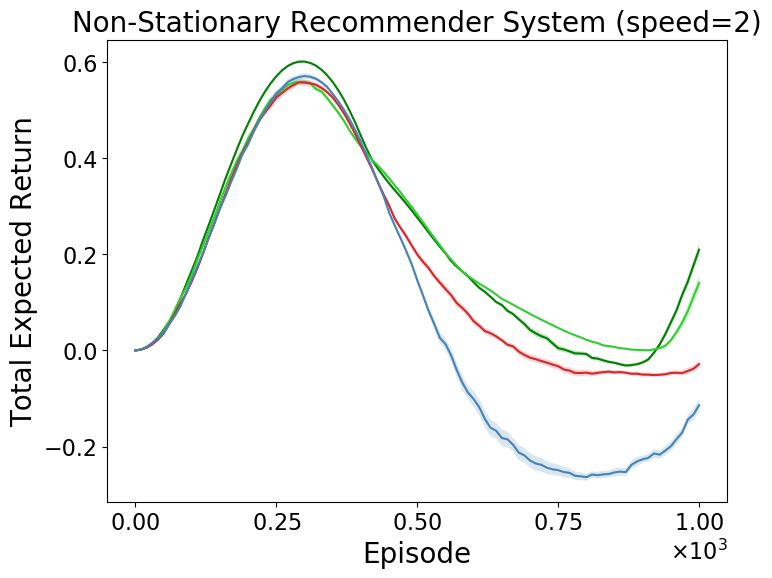}\\
    \includegraphics[width=0.3\textwidth]{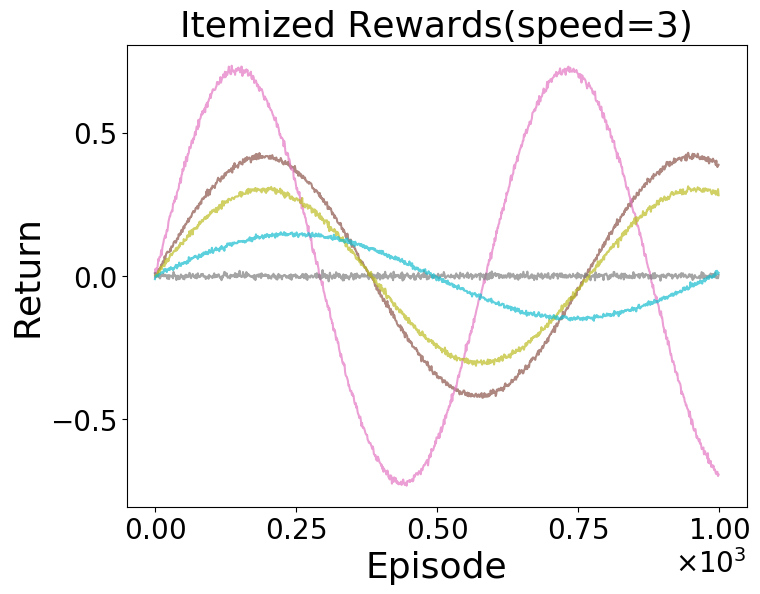} \quad\quad\quad\quad
    \includegraphics[width=0.3\textwidth]{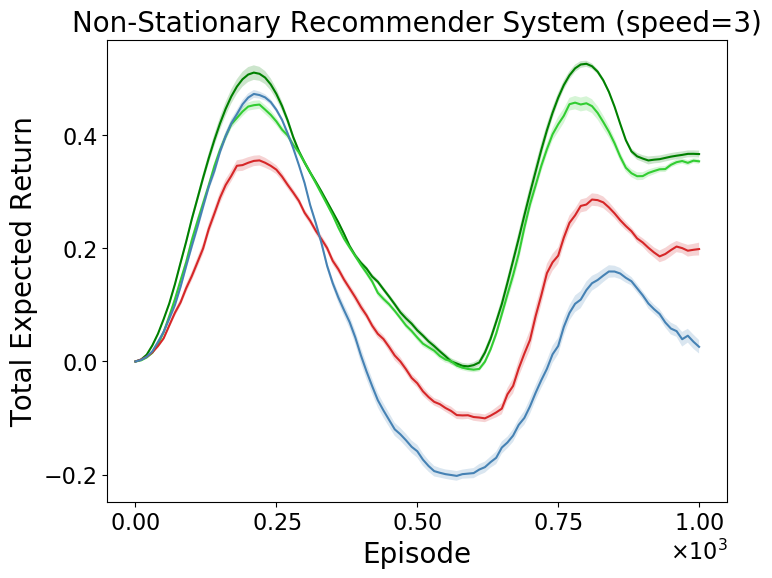}\\
    \includegraphics[width=0.3\textwidth]{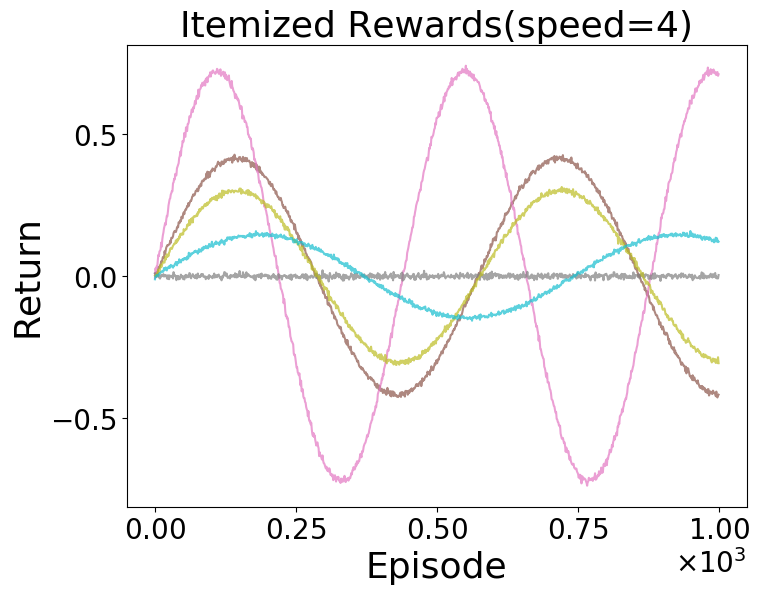} \quad\quad\quad\quad
    \includegraphics[width=0.3\textwidth]{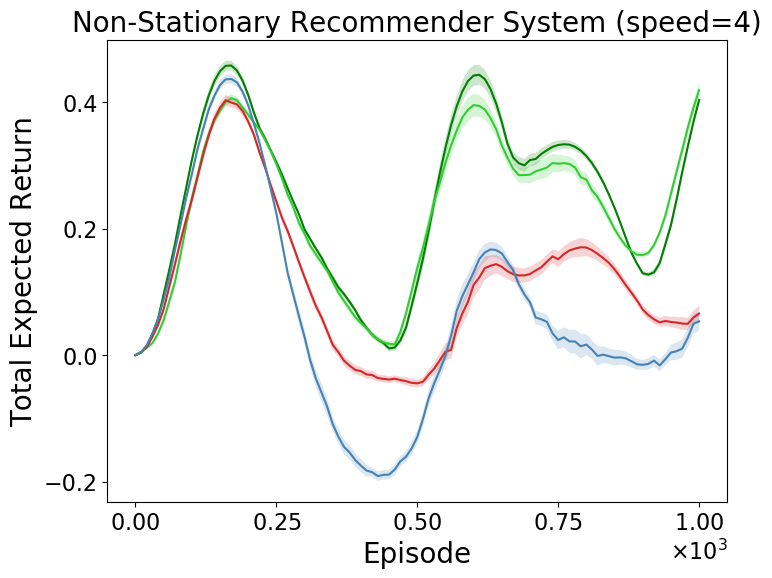}\\
    \quad\quad\quad\quad\quad\quad
    \includegraphics[width=0.35\textwidth]{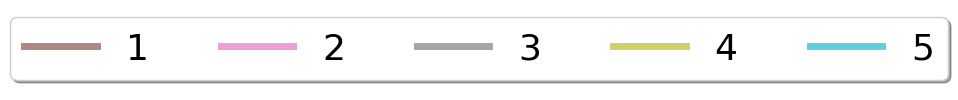} \quad\quad
    \includegraphics[width=0.4\textwidth]{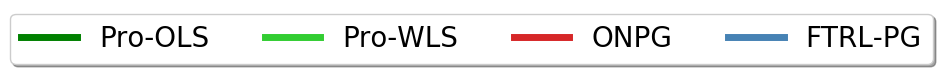}
    \caption{(Left) Fluctuations in the reward associated with each of the 5 items that can be recommended, for different speeds. (Right) Running mean of the best performance of all the algorithms for different speeds; higher total expected return is better. Shaded regions correspond to the standard error of the mean obtained using 30 trials. Notice the shape of the performance curve for the proposed methods, which closely captures the trend of maximum reward attainable over time.
}
    \label{apx:fig:recoresults}
\end{figure}

\begin{figure}
    \centering
    \includegraphics[width=0.3\textwidth]{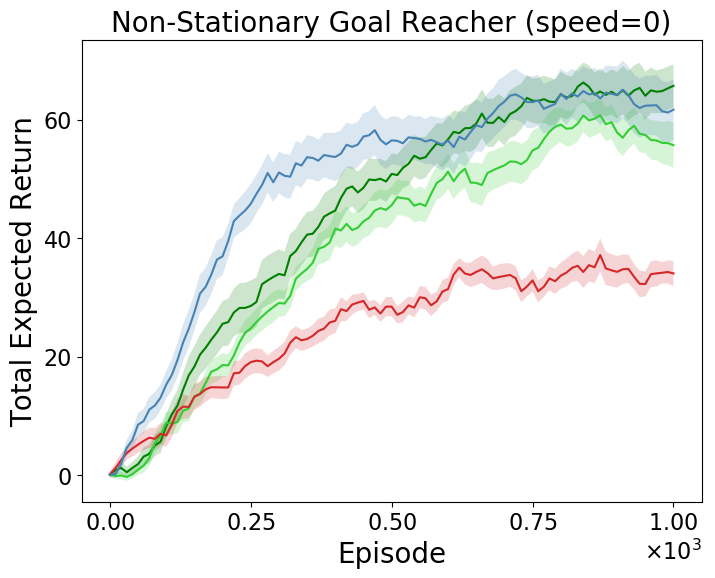} \quad \quad \quad\quad
    \includegraphics[width=0.3\textwidth]{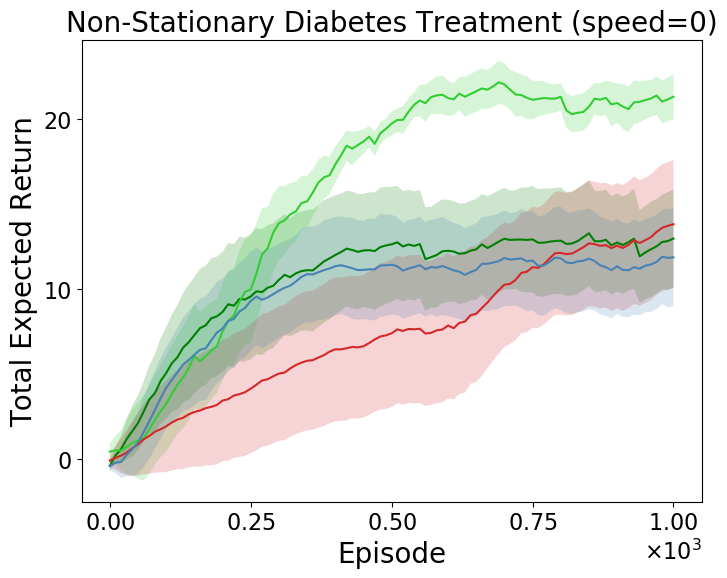}\\
    \includegraphics[width=0.3\textwidth]{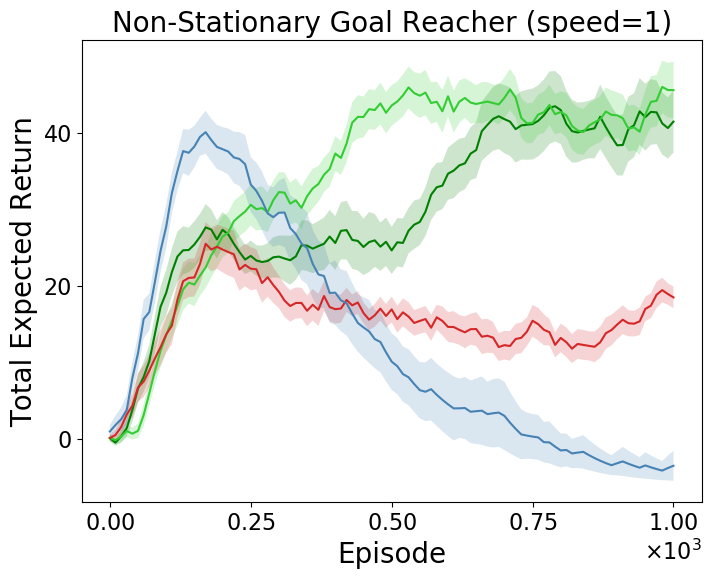} \quad \quad \quad\quad
    \includegraphics[width=0.3\textwidth]{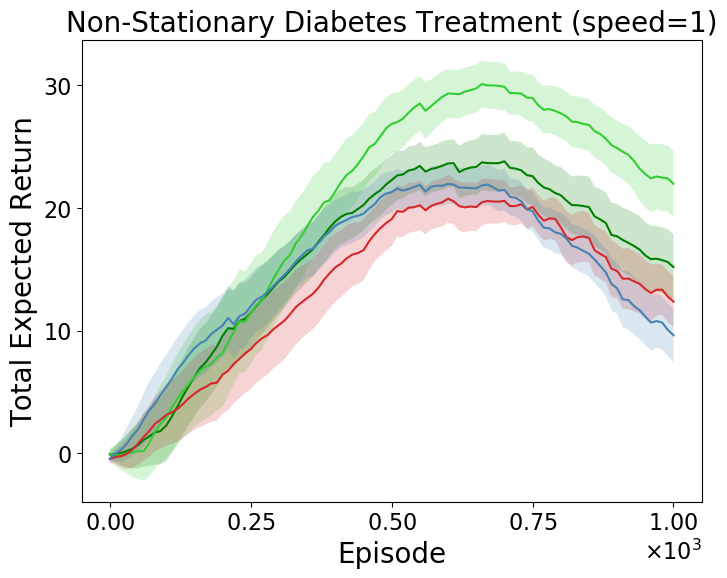}\\
    \includegraphics[width=0.3\textwidth]{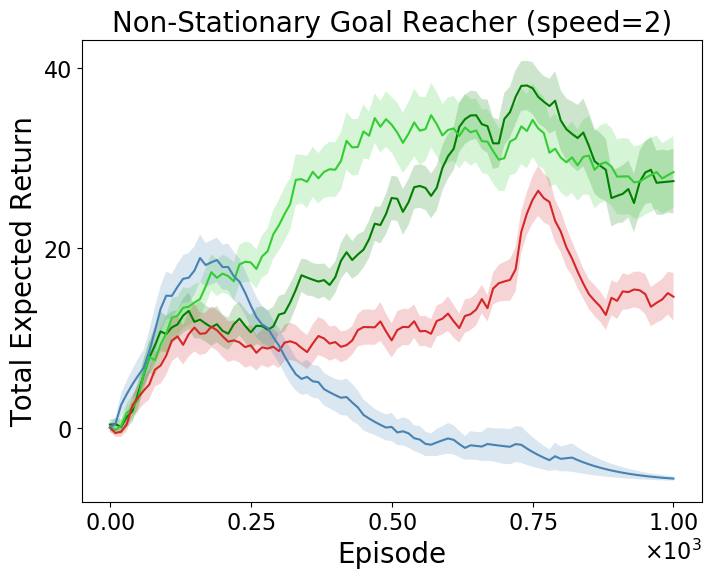} \quad \quad \quad\quad
    \includegraphics[width=0.3\textwidth]{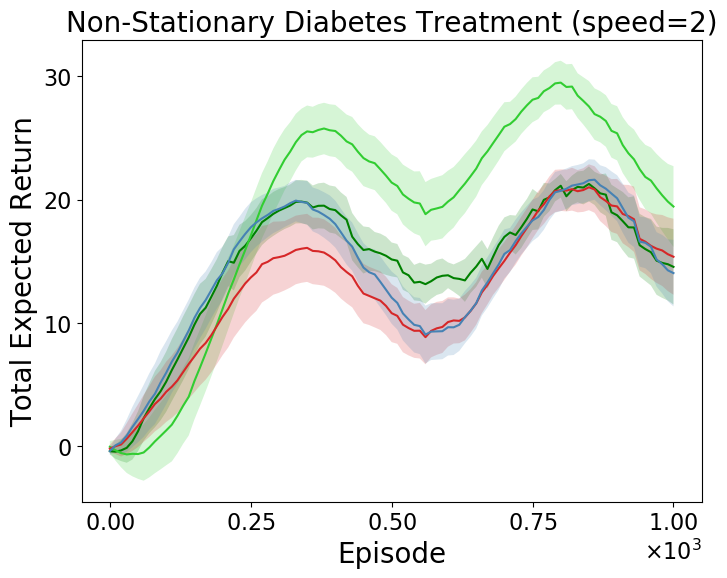}\\
    \includegraphics[width=0.3\textwidth]{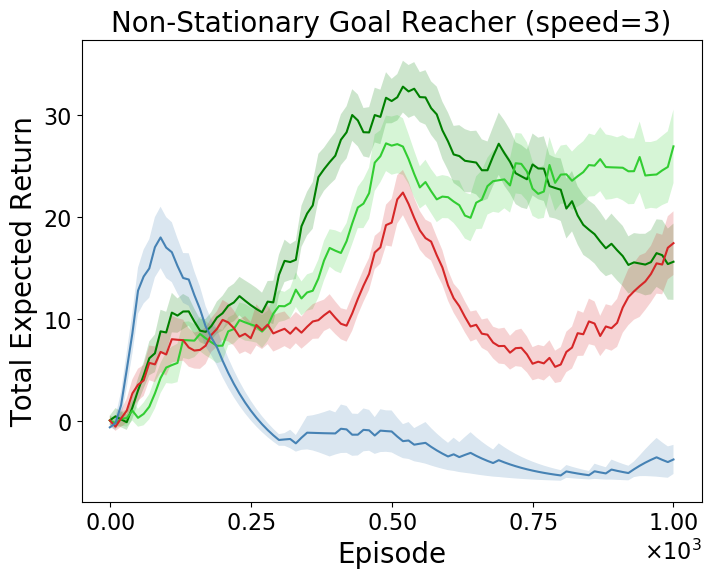} \quad \quad \quad\quad
    \includegraphics[width=0.3\textwidth]{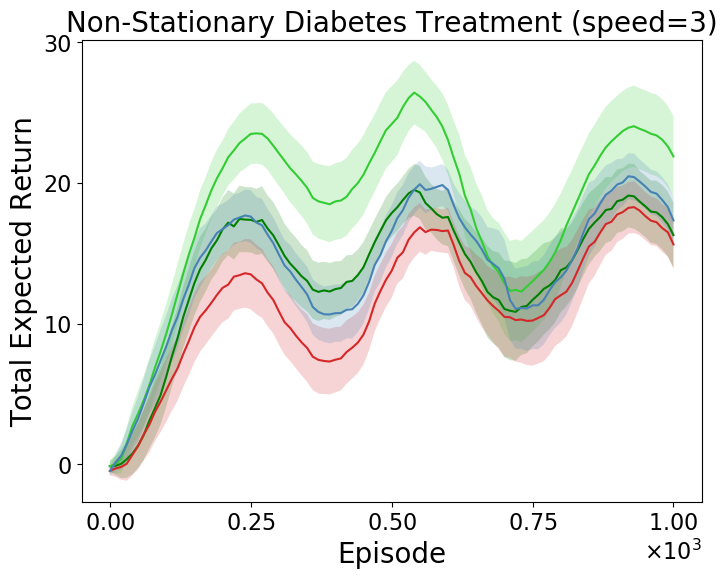}\\
    \includegraphics[width=0.3\textwidth]{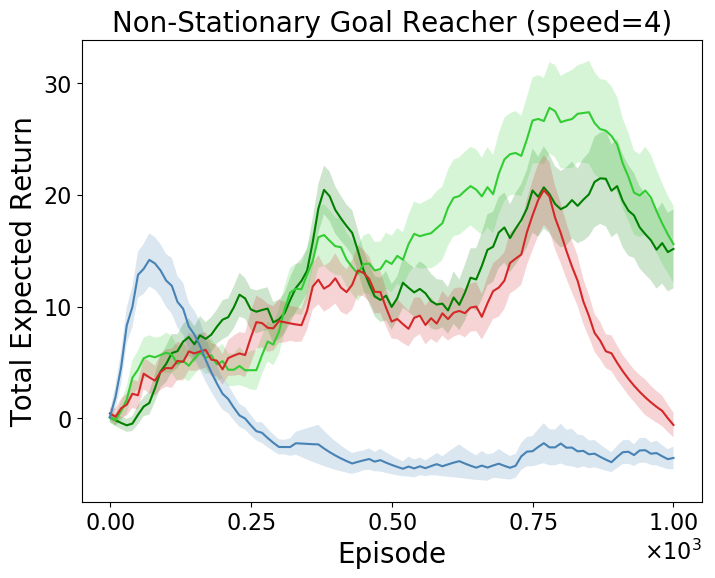} \quad \quad \quad\quad
    \includegraphics[width=0.3\textwidth]{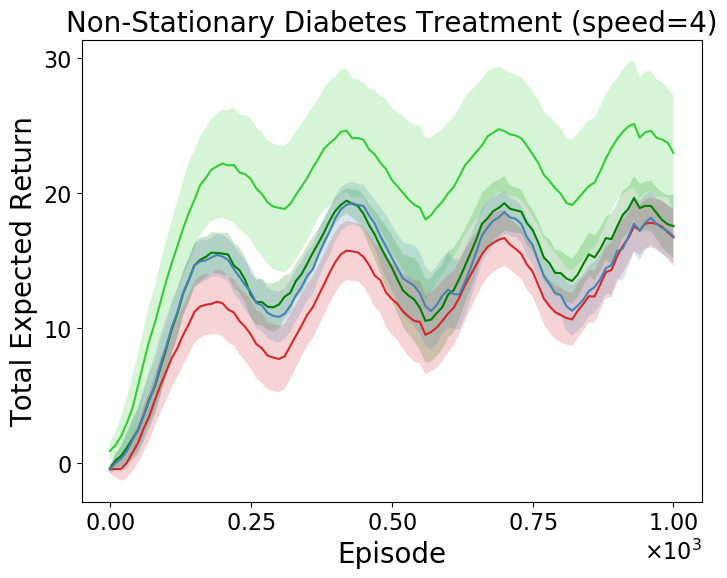}\\
    \includegraphics[width=0.5\textwidth]{Images/perf_legend.png}
    \caption{Running mean of the best performance of all the algorithms for different speeds; higher total expected return is better. Shaded regions correspond to the standard error of the mean obtained using 30 trials for NS Goal Reacher and 10 trials for NS Diabetes Treatment.}
    \label{apx:fig:detailedresults}
\end{figure}

\end{document}